\documentclass[lettersize,journal]{IEEEtran}
\usepackage{amsmath,amsfonts}
\usepackage{algorithmic}
\usepackage{algorithm}
\usepackage{array}
\usepackage[caption=false,font=normalsize,labelfont=sf,textfont=sf]{subfig}
\usepackage{textcomp}
\usepackage{stfloats}
\usepackage{url}
\usepackage{verbatim}
\usepackage{graphicx}
\usepackage[noadjust]{cite}
\usepackage{amssymb}
\usepackage{booktabs}
\usepackage{calc}
\usepackage{mathrsfs}
\usepackage{multirow}
\usepackage[rgb]{xcolor}

\everymath{\displaystyle}

\newtheorem{definition}{Definition}
\newtheorem{lemma}{Lemma}
\newtheorem{theorem}{Theorem}

\newcommand{\conttensor}[1]{\ensuremath{\boldsymbol{\mathsf{#1}}}}  % 连续张量
\newcommand{\mat}[1]{\ensuremath{\boldsymbol{#1}}}  % 矩阵
\newcommand{\operator}[1]{\ensuremath{\mathscr{#1}}}  % 算子
\newcommand{\tensor}[1]{\ensuremath{\boldsymbol{\mathcal{#1}}}}  % 张量
\DeclareMathOperator*{\argmin}{argmin}
\DeclareMathOperator{\fold}{fold}  % 张量折叠（展开的逆运算）
\newcolumntype{L}{@{\extracolsep{\fill}}l}
\newcolumntype{R}{@{\extracolsep{\fill}}r}
\newcolumntype{C}{@{\extracolsep{\fill}}c}

\begin{document}

\title{Neural Operator-Grounded Continuous Tensor Function Representation and Its Applications}

\author{Ruoyang Su, Xi-Le Zhao,~\IEEEmembership{Senior Member,~IEEE}, Sheng Liu, Wei-Hao Wu, Yisi Luo, and Michael K. Ng
        % <-this % stops a space
\thanks{%This research is supported by ...
\textit{(Corresponding author: Xi-Le Zhao)}}% <-this % stops a space
%\thanks{Manuscript received April 19, 2021; revised August 16, 2021.}
\thanks{Ruoyang Su, Xi-Le Zhao, Sheng Liu and Wei-Hao Wu are with the School of Mathematical Sciences, University of Electronic Science and Technology of China, Chengdu 611731, China (e-mail: 202511110508@std.uestc.edu.cn; xlzhao122003@163.com; liusheng16@163.com; weihaowu99@163.com).}
\thanks{Yisi Luo is with the School of Mathematics and Statistics, Xi'an Jiaotong University, Xi'an 710049, China (e-mail: yisiluo1221@foxmail.com).}
\thanks{Michael K. Ng is with the Department of Mathematics, Hong Kong Baptist University, Hong Kong, China (e-mail: michael-ng@hkbu.edu.hk).}
}

% The paper headers
%\markboth{Journal of \LaTeX\ Class Files,~Vol.~14, No.~8, August~2021}%
%{Shell \MakeLowercase{\textit{et al.}}: A Sample Article Using IEEEtran.cls for IEEE Journals}

%\IEEEpubid{0000--0000/00\$00.00~\copyright~2021 IEEE}
% Remember, if you use this you must call \IEEEpubidadjcol in the second
% column for its text to clear the IEEEpubid mark.

\maketitle

\begin{abstract}
Recently, continuous tensor functions have attracted increasing attention, because they can unifiedly represent data both on mesh grids and beyond mesh grids.
However, since mode-$n$ product is essentially discrete and linear, the potential of current continuous tensor function representations is still locked.
To break this bottleneck, we suggest neural operator-grounded mode-$n$ operators as a continuous and nonlinear alternative of discrete and linear mode-$n$ product.
Instead of mapping the discrete core tensor to the discrete target tensor, proposed mode-$n$ operator directly maps the continuous core tensor function to the continuous target tensor function, which provides a genuine continuous representation of real-world data and can ameliorate discretization artifacts.
Empowering with continuous and nonlinear mode-$n$ operators, we propose a neural operator-grounded continuous tensor function representation (abbreviated as NO-CTR), which can more faithfully represent complex real-world data compared with classic discrete tensor representations and continuous tensor function representations.
Theoretically, we also prove that any continuous tensor function can be approximated by NO-CTR.
To examine the capability of NO-CTR, we suggest an NO-CTR-based multi-dimensional data completion model. Extensive experiments across various data on regular mesh grids (multi-spectral images and color videos), on mesh girds with different resolutions (Sentinel-2 images) and beyond mesh grids (point clouds) demonstrate the superiority of NO-CTR.
\end{abstract}


\section{Introduction}

\begin{figure}[t]
	\centering
	\includegraphics[width=\linewidth]{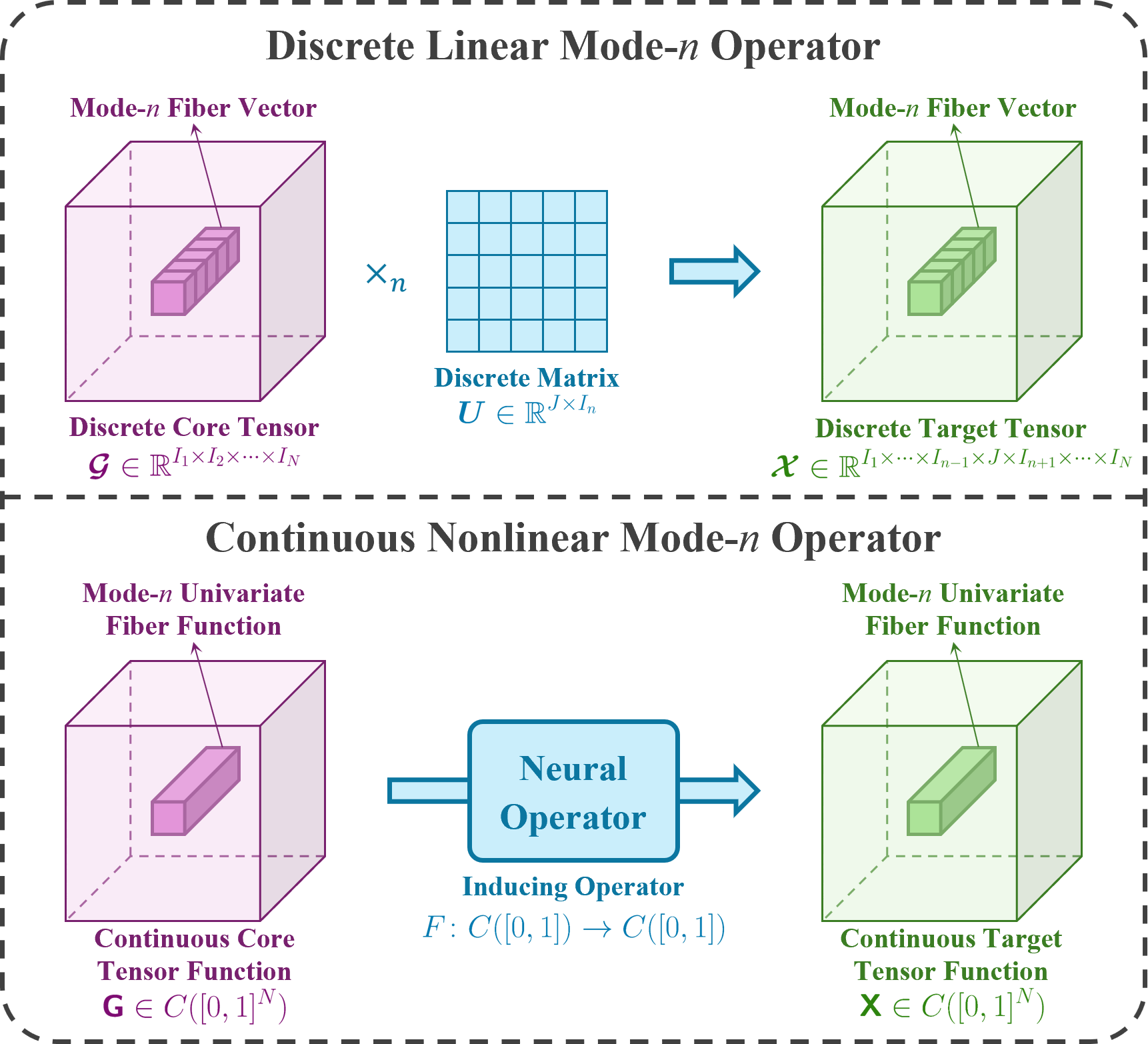}
	\caption{Discrete and linear mode-$n$ operator and proposed continuous and nonlinear mode-$n$ operator. The basic unit performed upon by discrete and linear mode-$n$ operator is the mode-$n$ fiber vectors of the discrete core tensor $\tensor{G}$, which are then mapped to the mode-$n$ fiber vectors of the discrete target tensor $\tensor{X}$. The basic unit performed upon by continuous and nonlinear mode-$n$ operator is the mode-$n$ univariate fiber functions of the continuous core tensor function $\conttensor{G}$, which are then mapped to the mode-$n$ univariate fiber functions of the continuous target tensor function $\conttensor{X}$.} \label{fig: motivation}
\end{figure}

\IEEEPARstart{M}{ulti-dimensional} data \cite{hou2022robust,chen2022bayesian,wang2023guaranteed,he2024multi} is the foundation of modern fields, and is widely used in various application scenarios, including remote sensing imaging \cite{he2022non,zhang2024full}, object detection \cite{li2023a,wang2025geometry,fu2025protd}, and signal processing \cite{zhang2022low,jiang2024multi,su2025deep}. Effectively representing and capturing the intrinsic structure of these multi-dimensional data is crucial for tasks such as classification \cite{salut2023randomized,zhang2025camgnet}, compression \cite{zhu2025implicit,shi2025dynamic}, and recovery \cite{xie2018kronecker,zhang2021low,xie2025irtf,xu2025to}. Simply flattening or vectorizing these multi-dimensional data matrices \cite{candes2009exact,cai2010a,fornasier2011low} often fails to maintain the high-order interactions, resulting in poor performance and interpretability. Therefore, developing effective methods to represent and process multi-dimensional data still remains a fundamental challenge.

As natural extensions of vectors and matrices to higher dimensions, tensors have become fundamental and common structures for representing multi-dimensional data. Traditional tensor methods mainly focus on low-rank decompositions.
The CANDECOMP/PARAFAC (CP) decomposition \cite{carroll1970analysis,harshman1970foundations} aims to decompose a tensor into the sum of rank-one components. The Tucker decomposition \cite{tucker1966some} focuses on decomposing a tensor into a core tensor multiplied by matrices along each mode. The tensor singular value decomposition (t-SVD) \cite{kilmer2008a,kilmer2011factorization,kilmer2013third} is designed to decompose a third-order tensor into the t-product of an f-diagonal tensor and two orthogonal tensors. The tensor train (TT) decomposition \cite{oseledets2011tensor} involves decomposing an $N$th-order tensor into a sequence of a matrix, $N-2$ third-order tensors, and another matrix. The tensor ring (TR) decomposition \cite{zhao2016tensor} is characterized by decomposing an $N$th-order tensor into $N$ third-order tensors multiplied in a circular manner.

The emergence of continuous tensor functions has revolutionized multi-dimensional data representation. By mapping the coordinates to corresponding data values, continuous tensor functions eliminate the dependency on fixed mesh grids, and seamlessly unify data on mesh grids and beyond mesh grids into a single flexible framework, providing a more flexible representation of multi-dimensional data.
Borrowing the wisdom of traditional tensor decompositions, researchers have developed a series of continuous tensor function representations.
Imaizumi et al. \cite{imaizumi2017tensor} proposed the smoothed Tucker decomposition (STD), which approximates an observed tensor by a small number of basis functions, and decomposes them through Tucker decomposition.
Fang et al. \cite{fang2024functional} proposed functional Bayesian Tucker decomposition (FunBaT). They treat the continuous-indexed data as the interaction between the Tucker core and a group of latent functions, and use Gaussian processes as functional priors to model the latent functions.
Furthermore, Luo et al. \cite{luo2024low} proposed a low-rank tensor function representation (LRTFR) parameterized by multi-layer perceptrons, which can continuously represent data beyond meshgrid with powerful representation abilities.
Despite this breakthrough, existing continuous tensor function representations are constrained by a fundamental limitation: the mappings (mode-$n$ product) from discrete core tensors to discrete target tensors are essentially discrete and linear, which renders them inadequate for capturing the complex structure of real-world data, leaving the potential of continuous tensor functions largely locked.

To break this bottleneck, in this paper, we suggest a continuous and nonlinear alternative of mode-$n$ product, i.e. continuous and nonlinear mode-$n$ operators.
To achieve this goal, we introduce neural operators \cite{anandkumar2019neural,lu2021learning,li2021fourier} into the field of continuous tensor function representations for the first time, to genuinely capture the complex structure of real-world data.
Concretely,  we leverage a neural operator to directly map the mode-$n$ univariate functions of the continuous core tensor function to those functions of the continuous target tensor function.
Then we propose a neural operator-grounded continuous tensor function representation (NO-CTR), to faithfully represent complex real-world data as a continuous core tensor function composite with a series of continuous and nonlinear mode-$n$ operators.
Figure \ref{fig: motivation} compares previous discrete and linear mode-$n$ operators and proposed continuous and nonlinear mode-$n$ operators.
Benefit from the powerful capability of continuous and nonlinear mode-$n$ operators to capture nonlinear relations, proposed NO-CTR can represent complex real-world data more faithfully.
NO-CTR not only unlocks the potential of continuous tensor functions, but also bridges the gap between neural operators and tensor representations.

The main contributions of this paper are as follows:

(1) To break the limitation of discrete and linear mode-$n$ product, we suggest neural operator-grounded mode-$n$ operators as a continuous and nonlinear alternative, which provides a genuine continuous representation of real-world data and can ameliorate discretization artifacts.

(2) Empowering with continuous and nonlinear mode-$n$ operators, we propose a neural operator-grounded continuous tensor function representation (NO-CTR), which can faithfully represent complex real-world data.

(3) Theoretically, we prove that any continuous tensor function can be approximated by NO-CTR. Besides, to examine the capability of NO-CTR, we suggest an NO-CTR-based multi-dimensional data completion model.

(4) Extensive experiments across various data on regular mesh grids (multi-spectral images and color videos), on mesh girds with different resolutions (Sentinel-2 images) and beyond mesh grids (point clouds) demonstrate the superiority of NO-CTR.

\section{Related Work}

\subsection{Continuous Tensor Function Representations}

One of the most popular continuous tensor function representations is implicit neural representation (INR), whose aim is to use a deep neural network to learn the continuous tensor function of the multi-dimensional data with respect to the coordinates.
It was originally used for 3D shape representation, such as \cite{park2019deepsdf,mescheder2019occupancy}.
After that, it was also widely used for representing various multi-dimensional data.
Sitzmann et al. \cite{sitzmann2020implicit} proposed to leverage periodic activation functions for implicit neural representations and demonstrated that these networks (called SIRENs) are ideally suited for representing complex natural signals and their derivatives.
Tancik et al. \cite{tancik2020fourier} suggested an approach for selecting problem-specific Fourier features that greatly improves the performance of  multi-layer perceptrons (MLPs) for low-dimensional regression tasks relevant to the computer vision and graphics communities.
Saragadam et al. \cite{saragadam2023wire} developed a new, highly accurate and robust INR: wavelet implicit neural representation (WIRE), a new INR based on a complex Gabor wavelet activation function.
Shi et al. \cite{shi2024improved} proposed a Fourier reparameterization method which learns coefficient matrix of fixed Fourier bases to compose the weights of multi-layer perceptron.
Recently, numerous researchers have extensive much work \cite{fathony2021multiplicative,chen2021llearning,yuce2022a,saragadam2022miner,li2025superpixel,zhou2025frequency} in this area.
INR provides a novel representation for multi-dimensional data, but it does not take into account the intrinsic structure and the interactions between factors of multi-dimensional data.

\subsection{Neural Operators}

Recently, neural operators have received extensive attention in scientific computing. In mathematics, a neural operator is a mapping from one function to another \cite{boulle2024a,kovachki2024operator}, with nonlinear structures and powerful capabilities. They have demonstrated great potential in solving partial differential equations usually by learning a mapping from the parameter function to the solution function \cite{kovachki2023neural}.
Anandkumar et al. \cite{anandkumar2019neural} formulated approximation of the infinite-dimensional mapping by composing nonlinear activation functions and a class of integral operators, where kernel integration is computed by message passing on graph networks.
Lu et al. \cite{lu2021learning} proposed a general deep learning framework, the deep operator network (DeepONet), to learn diverse continuous and nonlinear operators.
Li et al. \cite{li2021fourier} formulated a new neural operator by parameterizing the integral kernel directly in Fourier space, allowing for an expressive and efficient architecture.
In summary, neural operators are powerful tools usually with nonlinear structures, which can flexibly map from one function to another.

\section{Notations and Preliminaries}

Throughout this paper, scalars, vectors, matrices, and (discrete) tensors are denoted by $x$, $\mat{x}$, $\mat{X}$, and $\tensor{X}$, respectively. And $\tensor{X}(i_1,i_2,\cdots,i_N)$ denotes the $(i_1,i_2,\cdots,i_N)$th element of an $N$th-order discrete tensor $\tensor{X} \in \mathbb{R}^{I_1 \times I_2 \times \cdots \times I_N}$.

An $N$th-order continuous tensor function\footnote{There are other names of this concept in previous literature, such as continuous-indexed tensor \cite{fang2024functional}, functional tensor \cite{vemuri2025functional}, tensor function \cite{luo2024low}.}, denoted by $\conttensor{X}$, is an $N$-variable continuous function with a dense and complete domain, whose input is a coordinate with $N$ dimensions, and whose output is the corresponding value in the target data.
Without loss of generality, in this paper, we normalize the coordinates to $[0,1]$, so $\conttensor{X} \in C([0,1]^N)$.

Fibers are the higher-order analogue of matrix rows and columns. A fiber is defined by fixing every index but one \cite{kolda2009tensor}. For discrete tensor $\tensor{X} \in \mathbb{R}^{I_1 \times I_2 \times \cdots \times I_N}$, the mode-$n$ fiber vectors are denoted by
\[ \tensor{X}(i_1, i_2, \cdots, i_{n-1}, :, i_{n+1}, \cdots, i_N) \in \mathbb{R}^{I_n}. \]
Similarly, for continuous tensor function $\conttensor{X} \in C([0,1]^N)$, the mode-$n$ univariate fiber functions are denoted by
\[ \conttensor{X}(y_1, y_2, \cdots, y_{n-1}, \cdot, y_{n+1}, \cdots, y_N) \in C([0,1]). \]

\begin{definition}[Mode-$n$ Unfolding \cite{kolda2009tensor}]
	Given a discrete tensor $\tensor{X} \in \mathbb{R}^{I_1 \times I_2 \times \cdots \times I_N}$, the mode-$n$ unfolding, denoted by $\mat{X}_{(n)}$, arranges mode-$n$ fiber vectors to be the columns of the resulting matrix, which is defined as follows:
	\[ \mat{X}_{(n)}(i_n, j) = \tensor{X}(i_1, i_2, \cdots, i_n, \cdots, i_N), \]
	where $j = 1 + \sum_{\substack{k=1 \\ k \neq n}}^{N} (i_k-1) \prod_{\substack{m=1 \\ m \neq n}}^{k-1} I_m$, and $n=1,2,\cdots,N$. The inverse operation of mode-$n$ unfolding is denoted as $\fold_n$, i.e. $\tensor{X} = \fold_n(\mat{X}_{(n)})$.
\end{definition}

\section{Revisit of Mode-$n$ Product} \label{sec:mode-n product}

Previous tensor representations typically represent a tensor as a set of factors. The operations or connections between factors largely determine the representation capability.
In this section, we will revisit them from a new perspective---the perspective of operators.

Mode-$n$ product is widely used in tensor representations, which is defined as follows:
\begin{definition}[Mode-$n$ Product \cite{kolda2009tensor}] \label{df:mode-n product}
	Given an $N$th-order discrete tensor $\tensor{G} \in \mathbb{R}^{I_1 \times I_2 \times \cdots \times I_N}$ and a matrix $\mat{U} \in \mathbb{R}^{J \times I_n}$, the mode-$n$ product $\tensor{G} \times_n \mat{U} \in \mathbb{R}^{I_1 \times \cdots \times I_{n-1} \times J \times I_{n+1} \times \cdots \times I_N}$ is defined as
	\begin{align*}
		& \mathrel{\phantom{=}} \tensor{G} \times_n \mat{U} (i_1, i_2, \cdots, i_{n-1}, j, i_{n+1}, \cdots, i_N)\\
		& = \sum_{i_n=1}^{I_n} \tensor{G}(i_1, i_2, \cdots, i_{n-1}, i_n, i_{n+1}, \cdots, i_N) \mat{U}(j,i_n),
	\end{align*}
	which can also be equivalently expressed as
	\[ \tensor{G} \times_n \mat{U} = \fold_n (\mat{U} \mat{G}_{(n)}), \]
	where $n=1,2,\cdots,N$.
\end{definition}

To better understand the role of mode-$n$ product, we reinterpret it from the perspective of operators:

\begin{definition}[Discrete and Linear Mode-$n$ Operator]
	Given an $N$th-order discrete tensor $\tensor{G} \in \mathbb{R}^{I_1 \times I_2 \times \cdots \times I_N}$ and a matrix $\mat{U} \in \mathbb{R}^{J \times I_n}$, the discrete and linear mode-$n$ operator
	\[ \operator{U}^{[n]} \colon \mathbb{R}^{I_1 \times \cdots \times I_n \times \cdots \times I_N} \to \mathbb{R}^{I_1 \times \cdots \times I_{n-1} \times J \times I_{n+1} \times \cdots \times I_N} \]
	induced by $\mat{U}$ is defined as
	\begin{equation} \label{eq:discrete mode-n operator}
		\begin{split}
			& \mathrel{\phantom{=}} \operator{U}^{[n]} (\tensor{G}) (i_1, i_2, \cdots, i_{n-1}, j, i_{n+1}, \cdots, i_N) \\
			& = (\mat{U} \tensor{G}(i_1, i_2, \cdots, i_{n-1}, :, i_{n+1}, \cdots, i_N))(j),
		\end{split}
	\end{equation}
	where $n=1,2,\cdots,N$.
\end{definition}

Combined with Definition \ref{df:mode-n product}, it is easy to identify that discrete and linear mode-$n$ operators are equivalent to mode-$n$ product, i.e.
\[ \operator{U}^{[n]} (\tensor{G}) = \tensor{G} \times_n \mat{U}. \]
As shown in Figure \ref{fig: motivation}, the basic unit performed upon by discrete and linear mode-$n$ operator $\operator{U}^{[n]}$ is the mode-$n$ fiber vectors
\[ \tensor{G}(i_1, i_2, \cdots, i_{n-1}, :, i_{n+1}, \cdots, i_N) \in \mathbb{R}^{I_n} \]
of the discrete tensor $\tensor{G}$, which are then mapped to the mode-$n$ fiber vectors
\[ \operator{U}^{[n]} (\tensor{G}) (i_1, i_2, \cdots, i_{n-1}, :, i_{n+1}, \cdots, i_N) \in \mathbb{R}^{J} \]
of the discrete target tensor $\operator{U}^{[n]} (\tensor{G})$.
Therefore, in mode-$n$ product, the mapping from $\tensor{G}$ to $\operator{U}^{[n]} (\tensor{G})$ is inherently discrete and linear.

From the perspective of operators, we review the previous representations Tucker decomposition \cite{tucker1966some} and LRTFR \cite{luo2024low} in this perspective.
They both represent the discrete target tensor as a discrete core tensor composite with a set of discrete and linear mode-$n$ operators.
In mathematics, both of them can be reformulated as
\begin{equation} \label{eq:tucker}
	\tensor{X} = \operator{U}^{[N]} \circ \cdots \circ \operator{U}^{[2]} \circ \operator{U}^{[1]} (\tensor{G}),
\end{equation}
where $\tensor{G}$ is the discrete core tensor, and $\operator{U}^{[n]}$ is the discrete and linear mode-$n$ operator induced by factor matrix $\mat{U}_n$ ($n=1,2,\cdots,N$).
In Tucker decomposition \cite{tucker1966some}, $\{\mat{U}_n\}_{n=1}^N$ are ordinary matrices, while in LRTFR \cite{luo2024low}, $\{\mat{U}_n\}_{n=1}^N$ are matrices sampled from functions implemented by implicit neural representations.
Although the core tensor or factor matrices can by generated by neural networks, the mapping from discrete core tensor $\tensor{G}$ to discrete target tensor $\tensor{X}$ is essentially linear, which limits the capabilities of these representations. Besides, the real-world data are usually complex and contain generous nonlinear relations, which are also difficult to precisely capture merely through linear mappings.

\section{Neural Operator-Grounded Continuous Tensor Function Representation}

In this section, we will propose continuous and nonlinear mode-$n$ operators. Empowering with this, we will propose a neural operator-grounded continuous tensor function representation (NO-CTR), and will prove that any continuous tensor function can be approximated by NO-CTR. Lastly, we will suggest an NO-CTR-based multi-dimensional data completion model.

\subsection{Continuous and Nonlinear Mode-$n$ Operators}

To break the above limitations of discrete and linear mode-$n$ operators, in this subsection, we will propose continuous and nonlinear mode-$n$ operators.

In Section \ref{sec:mode-n product}, we revisited mode-$n$ product from the perspective of operators. In this perspective, the basic unit performed upon by discrete and linear mode-$n$ operator is the mode-$n$ fiber vectors.
This perspective inspires us to define new interactions for continuous tensor functions. Specifically, proposed continuous and nonlinear mode-$n$ operators directly perform on the mode-$n$ univariate fiber functions of a continuous tensor function, preserving this key property of discrete and linear mode-$n$ operators. Below we provide a definition of continuous and nonlinear mode-$n$ operators.

\begin{definition}[Continuous and Nonlinear Mode-$n$ Operator]
	Given an $N$th-order continuous tensor function $\conttensor{G} \in C([0,1]^N)$ and a continuous operator $F \colon C([0,1]) \to C([0,1])$, the continuous and nonlinear mode-$n$ operator
	\[ \operator{F}^{\langle n \rangle} \colon C([0,1]^N) \to C([0,1]^N) \]
	induced by $F$ is defined as
	\begin{equation} \label{eq:continuous mode-n operator}
		\begin{split}
			& \mathrel{\phantom{=}} \operator{F}^{\langle n \rangle} (\conttensor{G}) (y_1, y_2, \cdots, y_N) \\
			& = F \bigl( \conttensor{G}(y_1, y_2, \cdots, y_{n-1}, \cdot, y_{n+1}, \cdots, y_N) \bigr) (y_n),
		\end{split}
	\end{equation}
	where $n=1,2,\cdots,N$. To avoid confusion, we refer to $F$ as the inducing operator.
\end{definition}

In Equation \eqref{eq:continuous mode-n operator}, we start from a continuous tensor function $\conttensor{G}$ instead of a discrete tensor $\tensor{G}$.
As shown in Figure \ref{fig: motivation}, the basic unit performed upon by continuous and nonlinear mode-$n$ operator $\operator{F}^{\langle n \rangle}$ is the mode-$n$ univariate fiber functions
\[ \conttensor{G}(y_1, y_2, \cdots, y_{n-1}, \cdot, y_{n+1}, \cdots, y_N) \in C([0,1]) \]
of the continuous tensor function $\conttensor{G}$, which are then mapped to the mode-$n$ univariate fiber functions
\[ F \bigl( \conttensor{G}(y_1, y_2, \cdots, y_{n-1}, \cdot, y_{n+1}, \cdots, y_N) \bigr) \in C([0,1]) \]
of the continuous target tensor function $\operator{F}^{\langle n \rangle} (\conttensor{G})$. Besides, the inducing operator $F$ is usually nonlinear, so the continuous and nonlinear mode-$n$ operator is a nonlinear mapping from $\conttensor{G}$ to $\operator{F}^{\langle n \rangle} (\conttensor{G})$.

\begin{figure}[!tb]
	\centering \fontsize{6pt}{7pt} \selectfont \setlength{\tabcolsep}{0.1em}
	\begin{tabular*}{\linewidth}{@{}CCCC@{}}
		Observation & Discrete \& Linear & Continuous \& Nonlinear & Ground Truth \\
		\includegraphics[width=\linewidth/4-0.5ex]{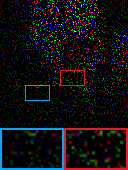} &
		\includegraphics[width=\linewidth/4-0.5ex]{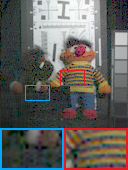} &
		\includegraphics[width=\linewidth/4-0.5ex]{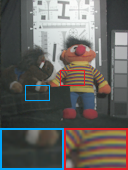} &
		\includegraphics[width=\linewidth/4-0.5ex]{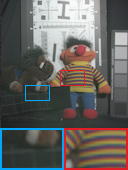} \\
		\parbox{6em}{\begin{tabular}{rl}
				PSNR: & 11.573 \\ SSIM: & 0.277
		\end{tabular}} & \parbox{6em}{\begin{tabular}{rl}
				PSNR: & 37.661 \\ SSIM: & 0.975
		\end{tabular}} & \parbox{6em}{\begin{tabular}{rl}
				PSNR: & 42.294 \\ SSIM: & 0.995
		\end{tabular}} & \parbox{6em}{\begin{tabular}{rl}
				PSNR: & Inf \\ SSIM: & 1.000
		\end{tabular}} \\
		\includegraphics[width=\linewidth/4-0.5ex]{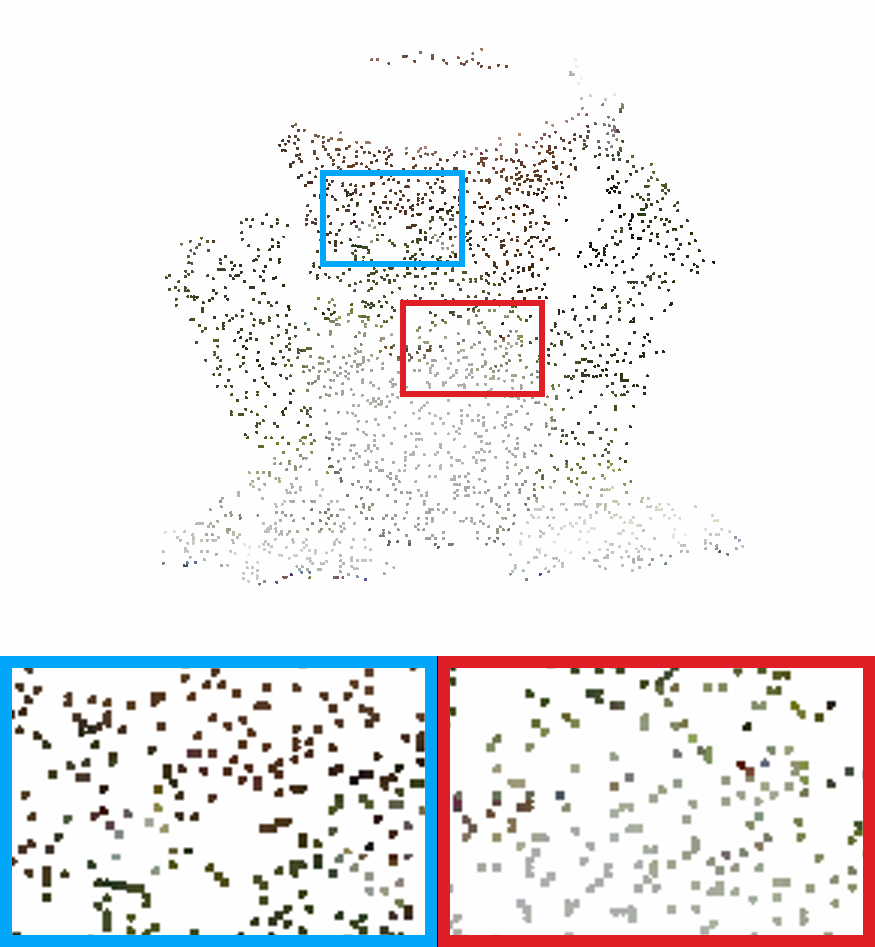} &
		\includegraphics[width=\linewidth/4-0.5ex]{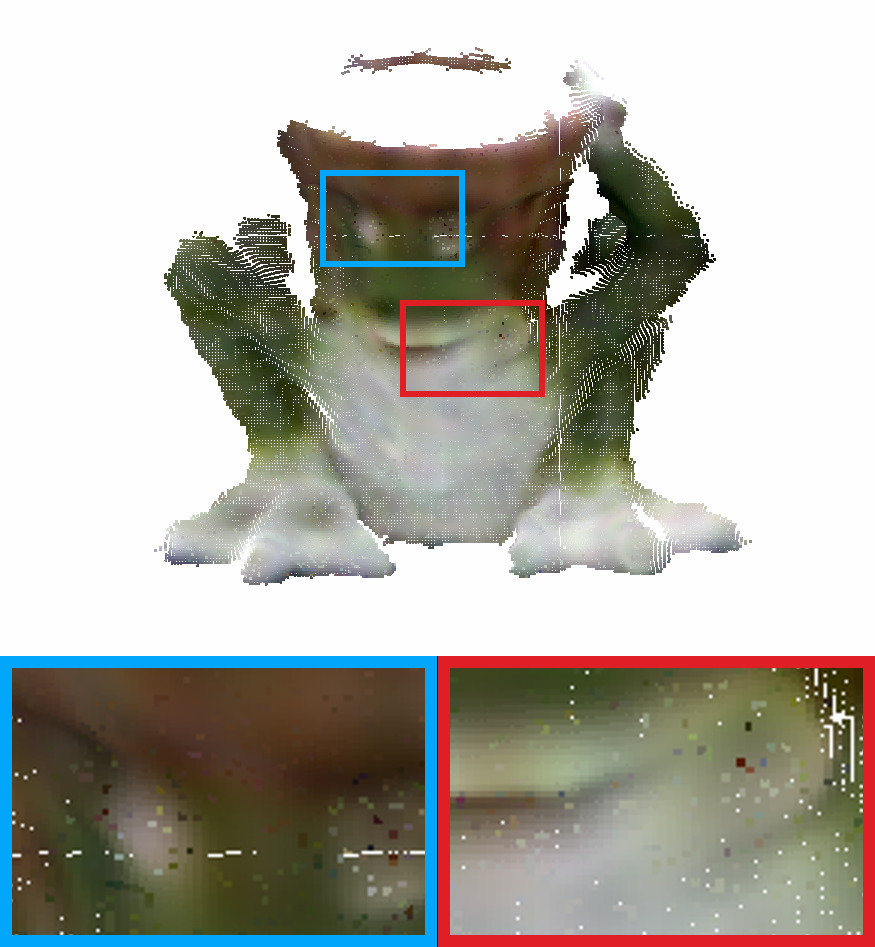} &
		\includegraphics[width=\linewidth/4-0.5ex]{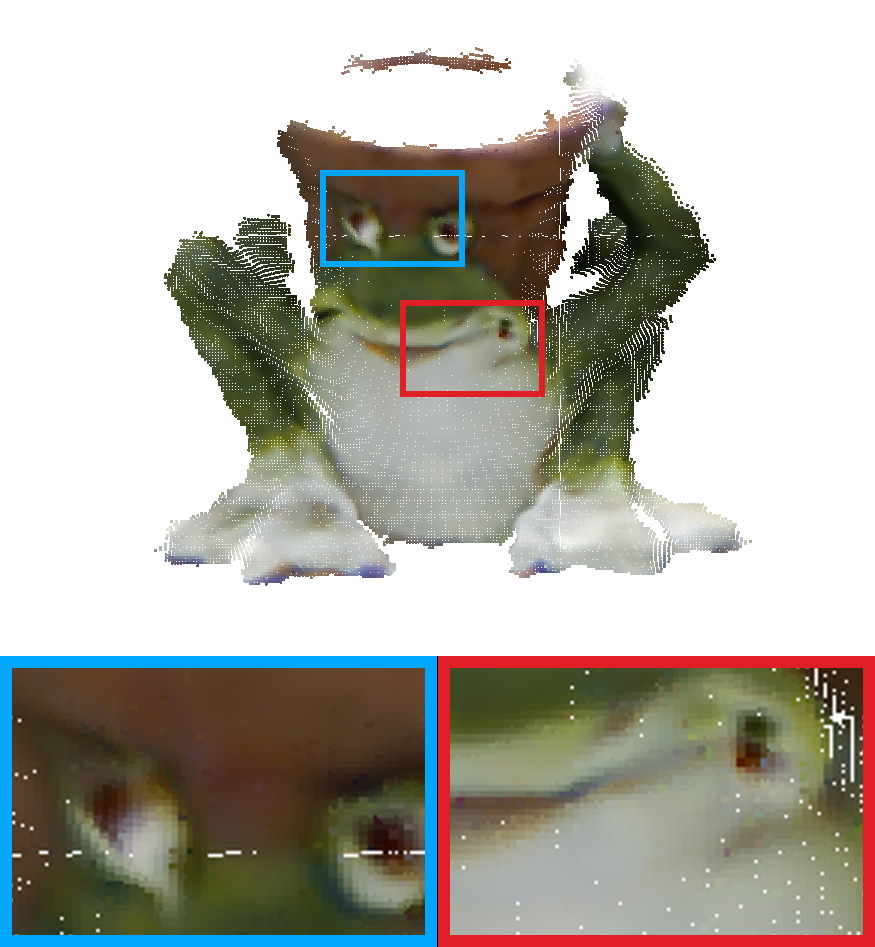} &
		\includegraphics[width=\linewidth/4-0.5ex]{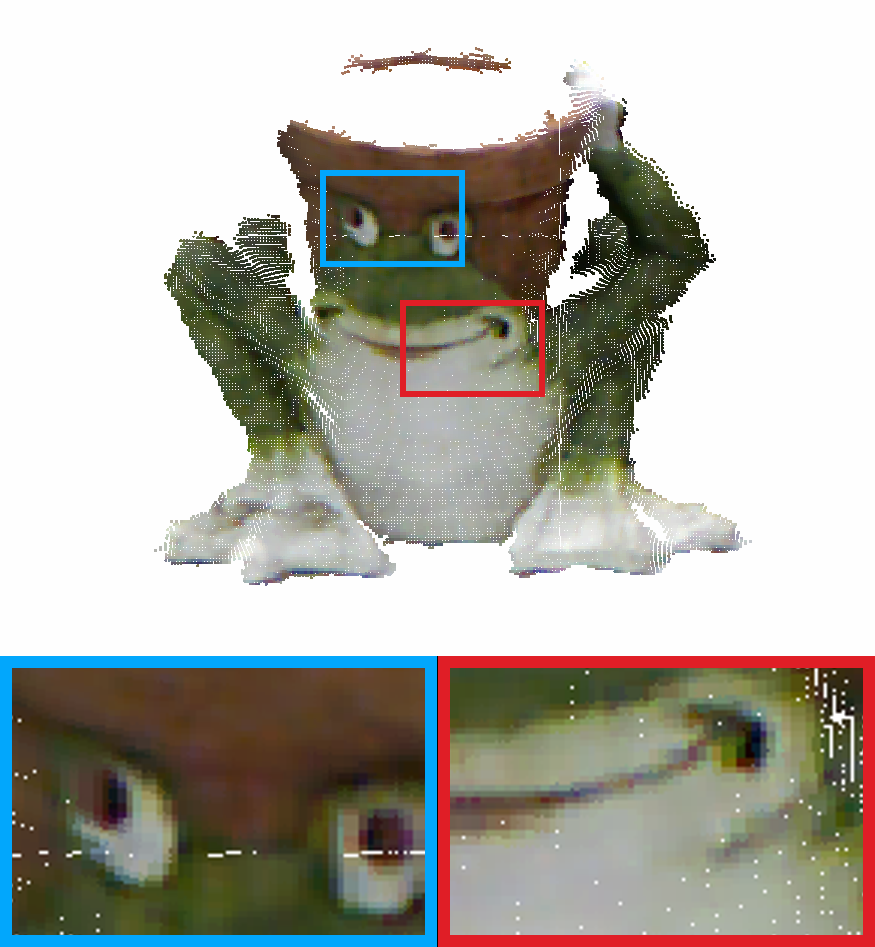} \\
		\parbox{6.5em}{\begin{tabular}{rl}
				NRMSE: & 0.457 \\ $R^2$: & -2.894
		\end{tabular}} & \parbox{6.5em}{\begin{tabular}{rl}
				NRMSE: & 0.054 \\ $R^2$: & 0.947
		\end{tabular}} & \parbox{6.5em}{\begin{tabular}{rl}
				NRMSE: & 0.041 \\ $R^2$: & 0.969
		\end{tabular}} & \parbox{6.5em}{\begin{tabular}{rl}
				NRMSE: & 0.000 \\ $R^2$: & 1.000
		\end{tabular}}
	\end{tabular*}
	\caption{Visual and quantitative recovery results with discrete and linear mode-$n$ operators (LRTFR \cite{luo2024low}), continuous and nonlinear mode-$n$ operators (proposed NO-CTR) on the MSI \textit{Toy} (top) and point cloud \textit{Frog} (bottom) when the sampling rates are $10\%$.} \label{fig:lrtfr vs ours}
\end{figure}

Finally, we use experimental results to further display the advantages of continuous and nonlinear operations.
Figure \ref{fig:lrtfr vs ours} shows the visual and quantitative recovery results with discrete and linear mode-$n$ operators (LRTFR \cite{luo2024low}), continuous and nonlinear mode-$n$ operators (proposed NO-CTR) when the sampling rates are $10\%$.
In terms of quantitative results, proposed NO-CTR outperforms LRTFR on various multi-dimensional data. In terms of visual results, proposed NO-CTR is capable of more accurately recovering details, such as the stripes on clothes and the eyes of frogs.
This reveals the superiority of proposed continuous and nonlinear mode-$n$ operators.

\subsection{Continuous Tensor Function Representation}

Empowering with continuous and nonlinear mode-$n$ operators, in this subsection, we will propose a neural operator-grounded continuous tensor function representation (NO-CTR).
Here we present a mathematical definition.

\begin{definition}[NO-CTR]
	For an $N$th-order continuous tensor function $\conttensor{X} \in C([0,1]^N)$, the neural operator-grounded continuous tensor function representation (abbreviated as NO-CTR) of $\conttensor{X}$ can be formulated as
	\begin{equation} \label{eq:no-ctr}
		\conttensor{X} = \operator{F}_N^{\langle N \rangle} \circ \cdots \circ \operator{F}_2^{\langle 2 \rangle} \circ \operator{F}_1^{\langle 1 \rangle} (\conttensor{G}),
	\end{equation}
	where $\conttensor{G} \in C([0,1]^N)$ is the continuous core tensor function, and $\operator{F}_n^{\langle n \rangle}$ is the continuous and nonlinear mode-$n$ operator induced by $F_n$ ($n=1,2,\cdots,N$).
\end{definition}

In Equation \eqref{eq:tucker}, the previous representations Tucker decomposition \cite{tucker1966some} and LRTFR \cite{luo2024low} represent the discrete target tensor $\tensor{X}$ as a discrete core tensor $\tensor{G}$ composite with a set of discrete and linear mode-$n$ operators $\{\operator{U}^{[n]}\}_{n=1}^N$.
Similarly, in Equation \eqref{eq:no-ctr}, proposed NO-CTR represents complex real-world data $\conttensor{X}$ as a continuous core tensor function $\conttensor{G}$ composite with a series of continuous and nonlinear mode-$n$ operators $\{\operator{F}^{\langle n \rangle}\}_{n=1}^N$.
Owing to the powerful capability of the continuous mode-$n$ operation to capture nonlinear relations, NO-CTR can represent complex real-world data more faithfully.

\subsection{Implementations}

\begin{figure*}[!t]
	\centering
	\includegraphics[width=\linewidth]{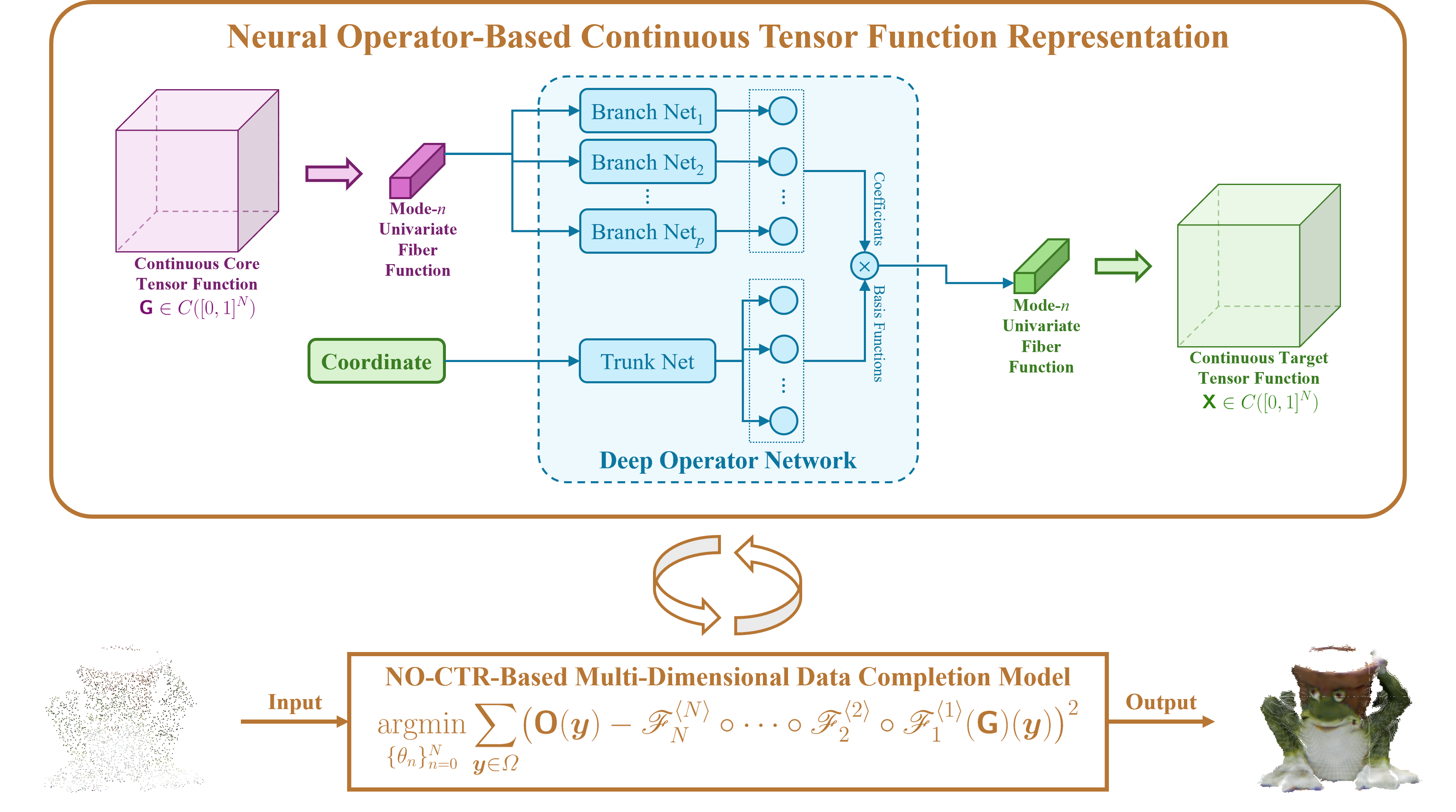}
	\caption{Flow chart of NO-CTR and the corresponding multi-dimensional data completion model. We first cleverly leverage neural operators to suggest continuous and nonlinear mode-$n$ operators. Specifically, the basic unit performed upon by continuous and nonlinear mode-$n$ operator is the mode-$n$ univariate fiber functions of the continuous core tensor function $\conttensor{G}$, which are then mapped to the mode-$n$ univariate fiber functions of the continuous target tensor function $\conttensor{X}$. Empowering with continuous and nonlinear mode-$n$ operators, we propose a neural operator-grounded continuous tensor function representation (NO-CTR), which can faithfully represent continuous target tensor function $\conttensor{X}$ as a continuous core tensor function $\conttensor{G}$ and a series of continuous and nonlinear mode-$n$ operators $\{\operator{F}^{\langle n \rangle}\}_{n=1}^N$. To examine the capability of NO-CTR, we suggest an NO-CTR-based multi-dimensional data completion model.}
\end{figure*}

In Equation \eqref{eq:no-ctr}, the continuous core tensor function $\conttensor{G}$ and continuous and nonlinear mode-$n$ operators $\{\operator{F}^{\langle n \rangle}\}_{n=1}^N$ are the key components of NO-CTR. In this subsection, we discuss the concrete implementations.

\subsubsection{Continuous Core Tensor Function}

The continuous core tensor function $\conttensor{G} \in C([0,1]^N)$ is essentially a multivariate continuous function defined on $[0,1]^N$.
In practice, neural networks are ideally suited for representing such continuous functions, and the choice of neural networks is highly diverse. In this paper, we choose the popular implicit neural representation SIREN \cite{sitzmann2020implicit} to implement this continuous core tensor function $\conttensor{G}$, which takes a coordinate $\mat{y} \in [0,1]^N$ as input and outputs its corresponding value $\conttensor{G}(\mat{y}) \in \mathbb{R}$.
We will also discuss other implementations of the continuous core tensor function in Section \ref{sec:contribution of continuous core tensor function}.

\subsubsection{Continuous and Nonlinear Mode-$n$ Operators}

In mathematics, any function operator can be utilized to implement the continuous and nonlinear mode-$n$ operator.
In this paper, we cleverly leverage an emerging and powerful tool: neural operators, which provide a correct and efficient choice for us.
Neural operators play a significant role in scientific computing \cite{anandkumar2019neural,lu2021learning,li2021fourier}. Essentially, a neural operator is a mapping from one function to another \cite{boulle2024a,kovachki2024operator}. This enables us to directly treat a neural operator as the inducing operator, marking the first application of neural operators in the field of tensor representations.

Among the diverse neural operators, we select the deep operator networks (DeepONet) \cite{luo2024low} to implement continuous and nonlinear mode-$n$ operators, and will discuss other architectures of neural operators in Section \ref{sec:contribution of neural operators}.
DeepONet is composed of the trunk network and branch networks. Specifically, for an inducing operator $F_n \colon C([0,1]) \to C([0,1])$,
\[ F_n(v)(y) = \sum_{p=1}^{P_n} b^{(n)}_p\bigl(v(z_1),v(z_2),\cdots,v(z_{m_n})\bigr) t^{(n)}_p(y), \]
where $v \in C([0,1])$ is the input univariate function, $y \in [0,1]$ is a coordinate.
$\mat{t}^{(n)}$ is a vector-valued function called trunk network, and $\{t^{(n)}_p\}_{p=1}^{P}$ are $P$ components of $\mat{t}$, which can be physically interpreted as the basis functions of the output space.
$\{b^{(n)}_p\}_{p=1}^{P_n}$ are scalar-valued functions called branch networks, which encode the discrete sampling values of input function $v$ at fixed sensor points $\{z_i\}_{i=1}^{m_n}$ (called sensors), providing the coefficients required for the linear combination.

\subsection{Universal Approximation Theorem}

To theoretically illustrate the rationality of the above implementations, we propose and prove the density of the NO-CTR, which indicates any continuous tensor function can be approximated by NO-CTR, and can be regarded as a universal approximation theorem.

\begin{theorem} \label{th:approx}
	For any continuous tensor function $\conttensor{X} \in C([0,1]^N)$, $\forall \varepsilon > 0$, there is a fully connected neural network
	\[ \conttensor{G} \colon [0,1]^N \to \mathbb{R} \]
	and $N$ DeepONets \cite{lu2021learning}
	\[ F_n \colon C([0,1]) \to C([0,1]), \quad (n=1,2,\cdots,N) \]
	such that
	\[ \| \conttensor{X} - \operator{F}_N^{\langle N \rangle} \circ \cdots \circ \operator{F}_2^{\langle 2 \rangle} \circ \operator{F}_1^{\langle 1 \rangle} (\conttensor{G}) \| < \varepsilon, \]
	where $\operator{F}_n$ is the continuous and nonlinear mode-$n$ operator induced by $F_n$.
\end{theorem}

The proof can be found in the appendix.

\subsection{Multi-Dimensional Data Completion Model} \label{sec:model}

To examine the capability of the NO-CTR, we suggest an NO-CTR-based multi-dimensional data completion model, which aims to estimate the missing elements from an incomplete observation.

In mathematics, this model can be formulated as
\begin{equation} \label{eq:O&F model}
	\argmin_{\{\theta_n\}_{n=0}^N} \sum_{\mat{y} \in \varOmega} \bigl( \conttensor{O}(\mat{y}) - \operator{F}_N^{\langle N \rangle} \circ \cdots \circ \operator{F}_2^{\langle 2 \rangle} \circ \operator{F}_1^{\langle 1 \rangle} (\conttensor{G}) (\mat{y}) \bigr)^2,
\end{equation}
where $\conttensor{O} \in C([0,1]^N)$ is the continuous tensor function of the observation, $\mat{y}$ is a coordinate, $\varOmega \subseteq [0,1]^N$ is a set composed of the coordinates of all observed elements, $\conttensor{G} \in C([0,1]^N)$ is the continuous core tensor function, $\operator{F}_n$ is the continuous and nonlinear mode-$n$ operator induced by $F_n$, $\theta_0$ denotes the parameters of the continuous core tensor function $\conttensor{G}$, and $\theta_n$ denotes the parameters of the inducing operator $F_n$ ($n=1,2,\cdots,N$).

\section{Experiments}

To comprehensively evaluate the performance of the NO-CTR, we conduct extensive experiments across various data on regular mesh grids (multi-spectral images and color videos), on mesh girds with different resolutions (Sentinel-2 images) and beyond mesh grids (point clouds).
In these experiments, observations are generated by randomly sampling the original ground truth at four different rates: $5\%$, $10\%$, $15\%$, and $20\%$. All methods use only one observation as training data in each experiment.

For data on mesh grids, the competing methods include both traditional tensor methods (TR-ALS \cite{wang2017efficient}) and continuous representation methods (SIREN \cite{sitzmann2020implicit}, MFN \cite{fathony2021multiplicative}, FR-INR \cite{shi2024improved} and LRTFR \cite{luo2024low}); for data beyond mesh grids, since traditional tensor methods cannot handle them, the competing methods only include continuous representation methods (SIREN \cite{sitzmann2020implicit}, MFN \cite{fathony2021multiplicative}, FR-INR \cite{shi2024improved} and LRTFR \cite{luo2024low}).
In terms of implementation, TR-ALS \cite{wang2017efficient} is coded in MATLAB, and the other methods are implemented based on the PyTorch library within the Anaconda environment.

The hyper-parameter configurations for all baseline methods strictly follow the settings reported in their original publications or the official source code provided by the authors.
For proposed NO-CTR, the hyper-parameters are adjusted based on the relevant literature \cite{sitzmann2020implicit,luo2024low} combined with extensive empirical tuning. Besides, considering the structure of multi-dimensional data and the limitation of computing power, in the experiments of this paper, we set all the continuous and nonlinear mode-$n$ operators along the spatial modes as identity operators.
During the training loop, quantitative evaluations are conducted every $100$ iterations by comparing the recovered output with the ground truth, and the checkpoint that yields the optimal performance is selected as the final result.

For an objective and quantitative comparison, four widely adopted metrics are employed to evaluate the reconstruction quality: peak signal-to-noise ratio (PSNR), structural similarity (SSIM) \cite{wang2004image}, normalized root mean square error (NRMSE), and coefficient of determination ($R^2$) as quality metrics. Generally, superior reconstruction quality is indicated by higher PSNR, SSIM, and $R^2$ values, as well as a lower NRMSE value.

\subsection{Multi-Spectral Images}

\begin{table*}[!tp]
	\caption{Quantitative Recovery Results of MSI Completion} \label{tab:msi}
	\begin{tabular*}{\linewidth}{CCCCCCCCCCCCCC@{}}
		\toprule
		\multicolumn{2}{c}{Sampling Rate} & \multicolumn{3}{c}{5\%} & \multicolumn{3}{c}{10\%} & \multicolumn{3}{c}{15\%} & \multicolumn{3}{c}{20\%} \\
		\cmidrule{1-2} \cmidrule{3-5} \cmidrule{6-8} \cmidrule{9-11} \cmidrule{12-14}
		MSI & Method & PSNR & SSIM & NRMSE & PSNR & SSIM & NRMSE & PSNR & SSIM & NRMSE & PSNR & SSIM & NRMSE \\
		\midrule
		\multirow{6}{*}{\includegraphics[height=6em]{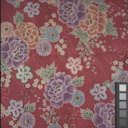}}
		& Observation & 15.532 & 0.045 & 0.274 & 15.769 & 0.078 & 0.266 & 16.023 & 0.112 & 0.259 & 16.277 & 0.148 & 0.251 \\
		& TR-ALS & 26.918 & 0.672 & 0.074 & 30.129 & 0.841 & 0.051 & 31.514 & 0.879 & 0.043 & 31.977 & 0.890 & 0.041 \\
		& SIREN & 29.737 & 0.842 & 0.053 & 34.377 & 0.946 & 0.031 & 37.860 & 0.982 & 0.021 & 41.742 & 0.992 & 0.013 \\
		& MFN & 28.507 & 0.806 & 0.061 & 31.843 & 0.907 & 0.042 & 36.741 & 0.975 & 0.024 & 40.022 & 0.989 & 0.016 \\
		& FR-INR & \underline{32.099} & \underline{0.918} & \underline{0.041} & \underline{37.273} & \underline{0.977} & \underline{0.022} & \underline{41.133} & \underline{0.991} & \underline{0.014} & \underline{43.701} & \underline{0.995} & \underline{0.011} \\
		& LRTFR & 30.631 & 0.859 & 0.048 & 34.300 & 0.946 & 0.032 & 37.817 & 0.977 & 0.021 & 40.461 & 0.988 & 0.016 \\
		\textit{Cloth} & NO-CTR & \textbf{33.906} & \textbf{0.943} & \textbf{0.033} & \textbf{39.188} & \textbf{0.985} & \textbf{0.018} & \textbf{42.793} & \textbf{0.994} & \textbf{0.012} & \textbf{45.579} & \textbf{0.997} & \textbf{0.009} \\
		\midrule
		\multirow{6}{*}{\includegraphics[height=6em]{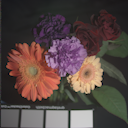}}
		& Observation & 16.269 & 0.372 & 0.168 & 16.497 & 0.403 & 0.163 & 16.744 & 0.436 & 0.159 & 17.010 & 0.468 & 0.154 \\
		& TR-ALS & 27.399 & 0.729 & 0.047 & 31.348 & 0.873 & 0.030 & 34.373 & 0.921 & 0.021 & 35.974 & 0.933 & 0.017 \\
		& SIREN & 33.706 & 0.925 & 0.023 & 38.681 & 0.977 & 0.013 & 41.702 & 0.989 & 0.009 & 46.041 & \underline{0.995} & 0.005 \\
		& MFN & 29.564 & 0.823 & 0.036 & 35.092 & 0.943 & 0.019 & 39.800 & 0.977 & 0.011 & 43.464 & 0.989 & 0.007 \\
		& FR-INR & \underline{34.713} & \underline{0.941} & \underline{0.020} & \underline{39.783} & \underline{0.981} & \underline{0.011} & \underline{43.309} & \underline{0.991} & \underline{0.007} & \underline{46.686} & 0.994 & \underline{0.005} \\
		& LRTFR & 33.619 & 0.930 & 0.023 & 36.879 & 0.965 & 0.016 & 40.100 & 0.984 & 0.011 & 43.249 & 0.991 & 0.008 \\
		\textit{Flowers} & NO-CTR & \textbf{37.121} & \textbf{0.981} & \textbf{0.015} & \textbf{42.591} & \textbf{0.994} & \textbf{0.008} & \textbf{46.456} & \textbf{0.997} & \textbf{0.005} & \textbf{49.717} & \textbf{0.999} & \textbf{0.004} \\
		\midrule
		\multirow{6}{*}{\includegraphics[height=6em]{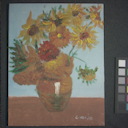}}
		& Observation & 16.574 & 0.105 & 0.258 & 16.810 & 0.141 & 0.251 & 17.059 & 0.177 & 0.244 & 17.328 & 0.214 & 0.237 \\
		& TR-ALS & 29.634 & 0.789 & 0.057 & 33.219 & 0.892 & 0.038 & 35.558 & 0.928 & 0.029 & 36.321 & 0.938 & 0.027 \\
		& SIREN & 34.747 & 0.918 & 0.032 & 39.731 & 0.973 & 0.018 & 43.023 & 0.986 & 0.012 & 46.164 & 0.994 & 0.009 \\
		& MFN & 32.563 & 0.863 & 0.041 & 37.103 & 0.943 & 0.024 & 41.731 & 0.980 & 0.014 & 44.520 & 0.990 & 0.010 \\
		& FR-INR & \underline{35.374} & \underline{0.934} & \underline{0.030} & \underline{40.407} & \underline{0.977} & \underline{0.017} & \underline{43.850} & \underline{0.989} & \underline{0.011} & \underline{46.953} & \underline{0.994} & \underline{0.008} \\
		& LRTFR & 34.597 & 0.925 & 0.032 & 38.807 & 0.967 & 0.020 & 41.587 & 0.982 & 0.014 & 44.254 & 0.989 & 0.011 \\
		\textit{Painting} & NO-CTR & \textbf{38.147} & \textbf{0.972} & \textbf{0.022} & \textbf{42.971} & \textbf{0.990} & \textbf{0.012} & \textbf{46.456} & \textbf{0.995} & \textbf{0.008} & \textbf{49.157} & \textbf{0.998} & \textbf{0.006} \\
		\midrule
		\multirow{6}{*}{\includegraphics[height=6em]{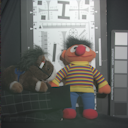}}
		& Observation & 11.332 & 0.241 & 0.278 & 11.573 & 0.277 & 0.270 & 11.826 & 0.314 & 0.263 & 12.085 & 0.351 & 0.255 \\
		& TR-ALS & 25.245 & 0.746 & 0.056 & 29.733 & 0.881 & 0.033 & 33.762 & 0.932 & 0.021 & 35.000 & 0.943 & 0.018 \\
		& SIREN & 31.916 & 0.915 & 0.026 & 37.259 & 0.974 & 0.014 & 42.048 & 0.988 & 0.008 & 45.306 & 0.995 & 0.006 \\
		& MFN & 28.157 & 0.802 & 0.040 & 34.213 & 0.922 & 0.020 & 38.931 & 0.969 & 0.012 & 42.759 & 0.987 & 0.007 \\
		& FR-INR & \underline{32.818} & 0.919 & \underline{0.023} & \underline{39.103} & 0.975 & \underline{0.011} & \underline{43.359} & \underline{0.989} & \underline{0.007} & \underline{46.464} & \underline{0.996} & \underline{0.005} \\
		& LRTFR & 31.836 & \underline{0.928} & 0.026 & 37.661 & \underline{0.975} & 0.013 & 40.447 & 0.986 & 0.010 & 43.099 & 0.991 & 0.007 \\
		\textit{Toy} & NO-CTR & \textbf{35.156} & \textbf{0.980} & \textbf{0.018} & \textbf{42.294} & \textbf{0.995} & \textbf{0.008} & \textbf{46.493} & \textbf{0.998} & \textbf{0.005} & \textbf{48.878} & \textbf{0.999} & \textbf{0.004} \\
		\bottomrule
	\end{tabular*}
	The \textbf{best} and \underline{second-best} results are highlighted.
\end{table*}

\begin{figure*}[!tp]
	\footnotesize
	\begin{tabular*}{\linewidth}{CCCCCCCC@{}}
		Observation & TR-ALS & SIREN & MFN & FR-INR & LRTFR & NO-CTR & Ground Truth \\
		\includegraphics[width=\linewidth/8-0.5ex]{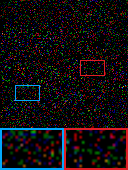} &
		\includegraphics[width=\linewidth/8-0.5ex]{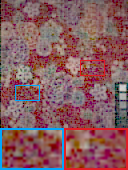} &
		\includegraphics[width=\linewidth/8-0.5ex]{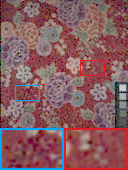} &
		\includegraphics[width=\linewidth/8-0.5ex]{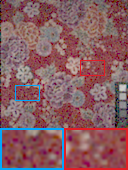} &
		\includegraphics[width=\linewidth/8-0.5ex]{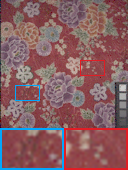} &
		\includegraphics[width=\linewidth/8-0.5ex]{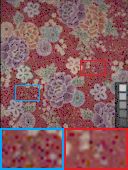} &
		\includegraphics[width=\linewidth/8-0.5ex]{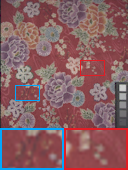} &
		\includegraphics[width=\linewidth/8-0.5ex]{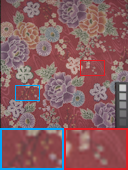} \\
		\includegraphics[width=\linewidth/8-0.5ex]{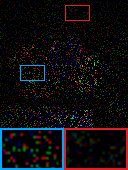} &
		\includegraphics[width=\linewidth/8-0.5ex]{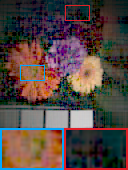} &
		\includegraphics[width=\linewidth/8-0.5ex]{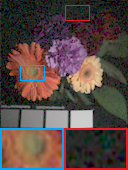} &
		\includegraphics[width=\linewidth/8-0.5ex]{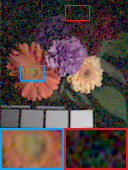} &
		\includegraphics[width=\linewidth/8-0.5ex]{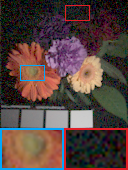} &
		\includegraphics[width=\linewidth/8-0.5ex]{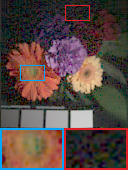} &
		\includegraphics[width=\linewidth/8-0.5ex]{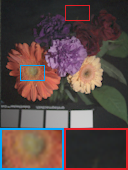} &
		\includegraphics[width=\linewidth/8-0.5ex]{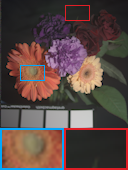} \\
		\includegraphics[width=\linewidth/8-0.5ex]{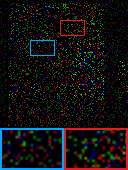} &
		\includegraphics[width=\linewidth/8-0.5ex]{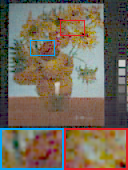} &
		\includegraphics[width=\linewidth/8-0.5ex]{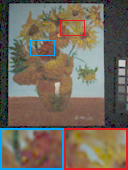} &
		\includegraphics[width=\linewidth/8-0.5ex]{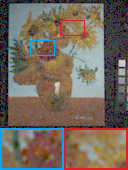} &
		\includegraphics[width=\linewidth/8-0.5ex]{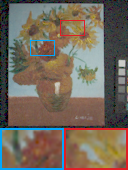} &
		\includegraphics[width=\linewidth/8-0.5ex]{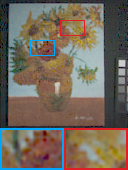} &
		\includegraphics[width=\linewidth/8-0.5ex]{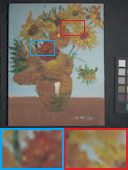} &
		\includegraphics[width=\linewidth/8-0.5ex]{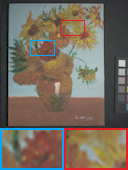} \\
		\includegraphics[width=\linewidth/8-0.5ex]{./pic/toy_obs.png} &
		\includegraphics[width=\linewidth/8-0.5ex]{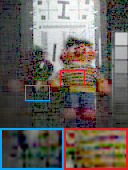} &
		\includegraphics[width=\linewidth/8-0.5ex]{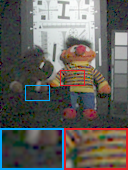} &
		\includegraphics[width=\linewidth/8-0.5ex]{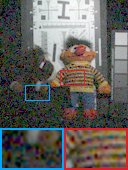} &
		\includegraphics[width=\linewidth/8-0.5ex]{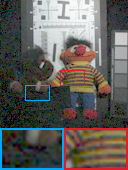} &
		\includegraphics[width=\linewidth/8-0.5ex]{./pic/toy_lrtfr.png} &
		\includegraphics[width=\linewidth/8-0.5ex]{./pic/toy_ours.png} &
		\includegraphics[width=\linewidth/8-0.5ex]{./pic/toy_gt.png}
	\end{tabular*}
	\caption{Visual recovery results of MSI completion at a sampling rate of $10\%$.} \label{fig:msi}
\end{figure*}

First, we conduct experiments on traditional and classic multi-dimensional data on regular mesh grids: multi-spectral images (MSI).
We use a database \cite{yasuma2008generalized} that spans a wide variety of real world materials.
Each multi-dimensional image is resized to $128 \times 128$ pixels with $31$ bands.

Table \ref{tab:msi} lists the quantitative recovery results of MSI completion. Proposed NO-CTR demonstrates superior performance across all experimental settings: it achieves the highest PSNR and SSIM values among all competing methods for all MSIs and all sampling rates. This advantage is consistent at low sampling rates, and becomes more pronounced as the sampling rates increase.
Figure \ref{fig:msi} shows the visual recovery results of MSI completion when the sampling rates are $10\%$. Proposed NO-CTR demonstrates remarkable performance in recovering fine details and textures. Compared with other methods, the results recovered by proposed NO-CTR are much closer to the ground truth.

\subsection{Color Videos}

\begin{table*}[!tp]
	\caption{Quantitative Recovery Results of Color Video Completion} \label{tab:color video}
	\begin{tabular*}{\linewidth}{CCCCCCCCCCCCCC@{}}
		\toprule
		\multicolumn{2}{c}{Sampling Rate} & \multicolumn{3}{c}{5\%} & \multicolumn{3}{c}{10\%} & \multicolumn{3}{c}{15\%} & \multicolumn{3}{c}{20\%} \\
		\cmidrule{1-2} \cmidrule{3-5} \cmidrule{6-8} \cmidrule{9-11} \cmidrule{12-14}
		Color Video & Method & PSNR & SSIM & NRMSE & PSNR & SSIM & NRMSE & PSNR & SSIM & NRMSE & PSNR & SSIM & NRMSE \\
		\midrule
		\multirow{6}{*}{\includegraphics[height=6em]{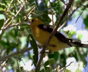}}
		& Observation & 6.947 & 0.015 & 0.436 & 7.180 & 0.032 & 0.425 & 7.434 & 0.056 & 0.412 & 7.684 & 0.081 & 0.401 \\
		& TR-ALS & 16.595 & 0.509 & 0.144 & 19.223 & 0.697 & 0.106 & 21.363 & 0.803 & 0.083 & 22.811 & 0.850 & 0.070 \\
		& SIREN & 18.877 & 0.690 & 0.110 & 21.631 & 0.825 & 0.080 & 23.401 & 0.875 & 0.066 & \underline{24.950} & \underline{0.914} & \underline{0.055} \\
		& MFN & 17.209 & 0.583 & 0.134 & 19.180 & 0.704 & 0.107 & 21.312 & 0.807 & 0.083 & 23.153 & 0.873 & 0.068 \\
		& FR-INR & \underline{19.337} & \underline{0.723} & \underline{0.105} & \underline{21.763} & \underline{0.826} & \underline{0.079} & \underline{23.519} & \underline{0.880} & \underline{0.065} & 24.737 & 0.904 & 0.056 \\
		& LRTFR & 18.340 & 0.661 & 0.118 & 20.705 & 0.777 & 0.090 & 22.437 & 0.850 & 0.073 & 23.423 & 0.870 & 0.065 \\
		\textit{Bird} & NO-CTR & \textbf{19.643} & \textbf{0.765} & \textbf{0.101} & \textbf{22.030} & \textbf{0.860} & \textbf{0.077} & \textbf{24.026} & \textbf{0.908} & \textbf{0.061} & \textbf{25.382} & \textbf{0.933} & \textbf{0.052} \\
		\midrule
		\multirow{6}{*}{\includegraphics[height=6em]{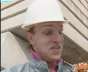}}
		& Observation & 3.895 & 0.007 & 0.676 & 4.128 & 0.013 & 0.658 & 4.375 & 0.019 & 0.639 & 4.641 & 0.025 & 0.620 \\
		& TR-ALS & 22.867 & 0.692 & 0.076 & 25.656 & 0.803 & 0.055 & 27.461 & 0.852 & 0.045 & 28.990 & 0.890 & 0.038 \\
		& SIREN & 24.450 & 0.805 & 0.063 & 27.119 & \underline{0.879} & 0.047 & 29.072 & 0.908 & 0.037 & 30.715 & 0.931 & 0.031 \\
		& MFN & 22.931 & 0.717 & 0.076 & 25.043 & 0.787 & 0.059 & 26.719 & 0.848 & 0.049 & 29.456 & 0.900 & 0.036 \\
		& FR-INR & \underline{24.955} & \underline{0.816} & \underline{0.060} & \underline{27.665} & 0.878 & \underline{0.044} & \underline{29.711} & \underline{0.910} & \underline{0.035} & \underline{31.421} & \underline{0.935} & \underline{0.028} \\
		& LRTFR & 24.712 & 0.789 & 0.062 & 27.267 & 0.862 & 0.046 & 28.818 & 0.895 & 0.038 & 29.853 & 0.914 & 0.034 \\
		\textit{Foreman} & NO-CTR & \textbf{26.759} & \textbf{0.891} & \textbf{0.049} & \textbf{29.719} & \textbf{0.940} & \textbf{0.035} & \textbf{31.841} & \textbf{0.957} & \textbf{0.027} & \textbf{32.940} & \textbf{0.967} & \textbf{0.024} \\
		\midrule
		\multirow{6}{*}{\includegraphics[height=6em]{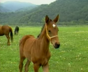}}
		& Observation & 7.219 & 0.011 & 0.425 & 7.451 & 0.025 & 0.414 & 7.699 & 0.043 & 0.402 & 7.978 & 0.063 & 0.389 \\
		& TR-ALS & 23.944 & 0.882 & 0.062 & 26.024 & 0.918 & 0.049 & 28.047 & 0.946 & 0.039 & 29.172 & 0.954 & 0.034 \\
		& SIREN & 25.705 & 0.911 & 0.051 & 27.739 & 0.942 & 0.040 & 29.193 & 0.957 & 0.034 & 30.444 & 0.964 & 0.029 \\
		& MFN & 24.984 & 0.899 & 0.055 & 26.897 & 0.927 & 0.044 & 28.632 & 0.944 & 0.036 & 30.012 & 0.964 & 0.031 \\
		& FR-INR & 26.384 & 0.928 & 0.047 & 28.326 & 0.946 & 0.037 & 29.831 & 0.958 & 0.031 & \underline{31.117} & 0.970 & \underline{0.027} \\
		& LRTFR & \underline{26.766} & \underline{0.932} & \underline{0.045} & \underline{28.622} & \underline{0.954} & \underline{0.036} & \underline{30.228} & \underline{0.966} & \underline{0.030} & 30.983 & \underline{0.971} & 0.028 \\
		\textit{Horse} & NO-CTR & \textbf{27.459} & \textbf{0.949} & \textbf{0.041} & \textbf{29.504} & \textbf{0.966} & \textbf{0.033} & \textbf{30.849} & \textbf{0.976} & \textbf{0.028} & \textbf{32.075} & \textbf{0.981} & \textbf{0.024} \\
		\bottomrule
	\end{tabular*}
	The \textbf{best} and \underline{second-best} results are highlighted.
\end{table*}

\begin{figure*}[!tp]
	\footnotesize
	\begin{tabular*}{\linewidth}{CCCCCCCC@{}}
		Observation & TR-ALS & SIREN & MFN & FR-INR & LRTFR & NO-CTR & Ground Truth \\
		\includegraphics[width=\linewidth/8-0.5ex]{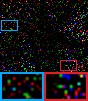} &
		\includegraphics[width=\linewidth/8-0.5ex]{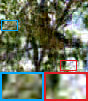} &
		\includegraphics[width=\linewidth/8-0.5ex]{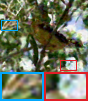} &
		\includegraphics[width=\linewidth/8-0.5ex]{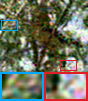} &
		\includegraphics[width=\linewidth/8-0.5ex]{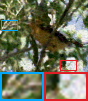} &
		\includegraphics[width=\linewidth/8-0.5ex]{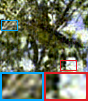} &
		\includegraphics[width=\linewidth/8-0.5ex]{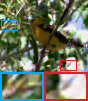} &
		\includegraphics[width=\linewidth/8-0.5ex]{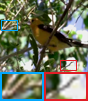} \\
		\includegraphics[width=\linewidth/8-0.5ex]{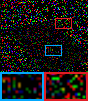} &
		\includegraphics[width=\linewidth/8-0.5ex]{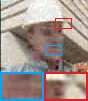} &
		\includegraphics[width=\linewidth/8-0.5ex]{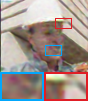} &
		\includegraphics[width=\linewidth/8-0.5ex]{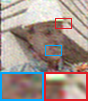} &
		\includegraphics[width=\linewidth/8-0.5ex]{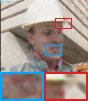} &
		\includegraphics[width=\linewidth/8-0.5ex]{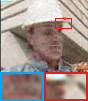} &
		\includegraphics[width=\linewidth/8-0.5ex]{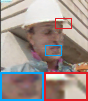} &
		\includegraphics[width=\linewidth/8-0.5ex]{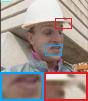} \\
		\includegraphics[width=\linewidth/8-0.5ex]{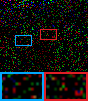} &
		\includegraphics[width=\linewidth/8-0.5ex]{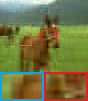} &
		\includegraphics[width=\linewidth/8-0.5ex]{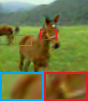} &
		\includegraphics[width=\linewidth/8-0.5ex]{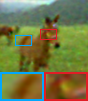} &
		\includegraphics[width=\linewidth/8-0.5ex]{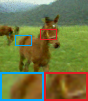} &
		\includegraphics[width=\linewidth/8-0.5ex]{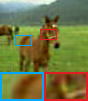} &
		\includegraphics[width=\linewidth/8-0.5ex]{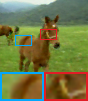} &
		\includegraphics[width=\linewidth/8-0.5ex]{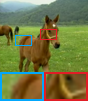}
	\end{tabular*}
	\caption{Visual recovery results of color video completion at a sampling rate of $10\%$.} \label{fig:color video}
\end{figure*}

Then, we conduct experiments on higher-order multi-dimensional data on regular mesh grids: color videos, which consist of a temporally ordered sequence of frames.
Each color video is resized to $144 \times 176$ pixels with three color channels (RGB) and consists of 30 consecutive frames.

Table \ref{tab:color video} lists the quantitative recovery results of color video completion. Proposed NO-CTR consistently outperforms all competing methods across all color videos and all sampling rates, achieving the highest PSNR and SSIM values in every experimental setting. This superiority is evident at high sampling rates, and it maintains this lead at low sampling rates.
Figure \ref{fig:color video} shows the visual recovery results of color video completion when the sampling rates are $10\%$. Proposed NO-CTR effectively restores the visual details of dynamic scenes. In terms of these color videos, proposed NO-CTR outperforms competing methods by recovering clearer textures and sharper edges, thus better retaining the visual information of the color videos.

\subsection{Sentinel-2 Images}

\begin{table*}[!tp]
	\caption{Quantitative Recovery Results of Sentinel-2 Image Completion} \label{tab:sentinel-2}
	\begin{tabular*}{\linewidth}{CCCCCCCCCCCCCC@{}}
		\toprule
		\multicolumn{2}{c}{Sampling Rate} & \multicolumn{3}{c}{5\%} & \multicolumn{3}{c}{10\%} & \multicolumn{3}{c}{15\%} & \multicolumn{3}{c}{20\%} \\
		\cmidrule{1-2} \cmidrule{3-5} \cmidrule{6-8} \cmidrule{9-11} \cmidrule{12-14}
		Dataset & Method & PSNR & SSIM & NRMSE & PSNR & SSIM & NRMSE & PSNR & SSIM & NRMSE & PSNR & SSIM & NRMSE \\
		\midrule
		\multirow{6}{*}{\includegraphics[height=6em]{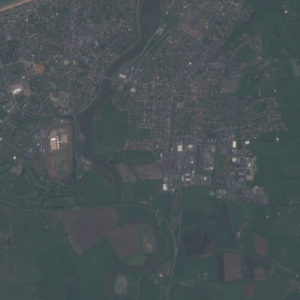}}
		& Observation & 9.512 & 0.015 & 0.454 & 9.746 & 0.022 & 0.443 & 9.998 & 0.027 & 0.430 & 10.267 & 0.031 & 0.417 \\
		& TR-ALS & 25.993 & 0.508 & 0.082 & 29.151 & 0.567 & 0.071 & 30.638 & 0.688 & 0.059 & 32.071 & 0.705 & 0.057 \\
		& SIREN & 32.382 & 0.715 & 0.056 & 33.999 & 0.770 & 0.049 & 35.197 & 0.834 & 0.038 & 36.280 & 0.862 & 0.036 \\
		& MFN & 32.202 & 0.734 & 0.058 & 33.814 & 0.800 & 0.046 & 34.967 & 0.848 & 0.039 & 35.904 & 0.871 & 0.034 \\
		& FR-INR & \textbf{33.243} & \underline{0.777} & \underline{0.050} & \underline{34.673} & \underline{0.833} & \underline{0.041} & \underline{35.288} & \underline{0.868} & \underline{0.035} & \underline{36.499} & \underline{0.890} & 0.031 \\
		& LRTFR & 30.534 & 0.706 & 0.060 & 32.039 & 0.809 & 0.043 & 33.240 & 0.861 & 0.036 & 34.227 & 0.882 & \underline{0.031} \\
		\textit{T30UYV} & NO-CTR & \underline{33.206} & \textbf{0.822} & \textbf{0.045} & \textbf{35.231} & \textbf{0.873} & \textbf{0.037} & \textbf{36.389} & \textbf{0.900} & \textbf{0.032} & \textbf{37.077} & \textbf{0.916} & \textbf{0.029} \\
		\midrule
		\multirow{6}{*}{\includegraphics[height=6em]{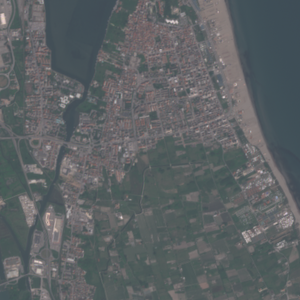}}
		& Observation & 8.587 & 0.014 & 0.485 & 8.819 & 0.019 & 0.472 & 9.078 & 0.024 & 0.458 & 9.340 & 0.028 & 0.445 \\
		& TR-ALS & 22.935 & 0.461 & 0.101 & 27.835 & 0.608 & 0.061 & 30.139 & 0.696 & 0.051 & 31.615 & 0.691 & 0.045 \\
		& SIREN & 32.530 & 0.698 & 0.052 & 33.839 & 0.790 & 0.039 & 34.834 & 0.821 & 0.036 & 36.026 & 0.858 & \underline{0.027} \\
		& MFN & 31.518 & 0.726 & 0.053 & 33.345 & 0.796 & 0.041 & 34.575 & 0.844 & 0.036 & 35.604 & 0.878 & 0.030 \\
		& FR-INR & \underline{32.974} & \underline{0.792} & \underline{0.043} & \underline{34.507} & \underline{0.849} & \underline{0.036} & \underline{35.381} & \underline{0.874} & \underline{0.031} & \underline{36.299} & \underline{0.891} & 0.028 \\
		& LRTFR & 30.111 & 0.745 & 0.047 & 31.740 & 0.824 & 0.037 & 33.043 & 0.859 & 0.033 & 33.971 & 0.884 & 0.029 \\
		\textit{T32TQR} & NO-CTR & \textbf{33.068} & \textbf{0.829} & \textbf{0.040} & \textbf{35.100} & \textbf{0.886} & \textbf{0.031} & \textbf{36.287} & \textbf{0.911} & \textbf{0.027} & \textbf{36.971} & \textbf{0.928} & \textbf{0.024} \\
		\midrule
		\multirow{6}{*}{\includegraphics[height=6em]{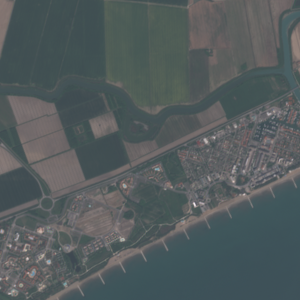}}
		& Observation & 8.883 & 0.015 & 0.479 & 9.118 & 0.021 & 0.467 & 9.357 & 0.026 & 0.454 & 9.633 & 0.032 & 0.440 \\
		& TR-ALS & 24.291 & 0.509 & 0.085 & 27.084 & 0.573 & 0.061 & 29.234 & 0.660 & 0.053 & 31.625 & 0.707 & 0.044 \\
		& SIREN & 32.547 & 0.765 & \underline{0.043} & 34.807 & 0.833 & \underline{0.033} & \underline{36.176} & 0.867 & \underline{0.027} & \underline{37.280} & 0.882 & \underline{0.023} \\
		& MFN & 31.761 & 0.769 & 0.050 & 33.921 & 0.839 & 0.038 & 34.898 & 0.869 & 0.033 & 36.537 & 0.901 & 0.027 \\
		& FR-INR & \underline{33.329} & \underline{0.820} & 0.043 & \underline{34.955} & \underline{0.871} & 0.033 & 35.812 & \underline{0.897} & 0.027 & 36.937 & 0.908 & 0.024 \\
		& LRTFR & 30.691 & 0.797 & 0.044 & 32.433 & 0.846 & 0.036 & 33.681 & 0.884 & 0.030 & 34.611 & \underline{0.909} & 0.026 \\
		\textit{T33TUL} & NO-CTR & \textbf{33.673} & \textbf{0.879} & \textbf{0.036} & \textbf{35.892} & \textbf{0.920} & \textbf{0.028} & \textbf{37.136} & \textbf{0.938} & \textbf{0.024} & \textbf{38.100} & \textbf{0.946} & \textbf{0.021} \\
		\bottomrule
	\end{tabular*}
	The \textbf{best} and \underline{second-best} results are highlighted.
\end{table*}

\begin{figure*}[!tp]
	\footnotesize
	\begin{tabular*}{\linewidth}{CCCCCCCC@{}}
		Observation & TR-ALS & SIREN & MFN & FR-INR & LRTFR & NO-CTR & Ground Truth \\
		\includegraphics[width=\linewidth/8-0.5ex]{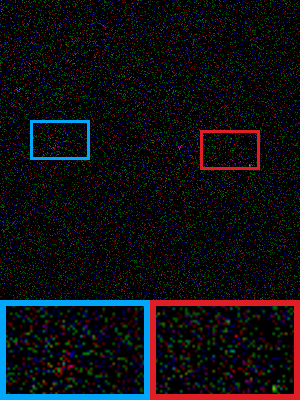} &
		\includegraphics[width=\linewidth/8-0.5ex]{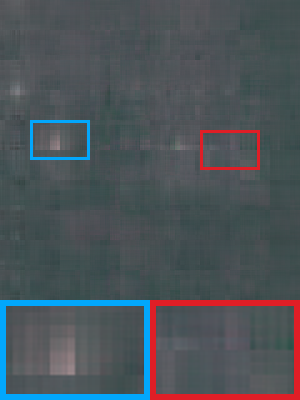} &
		\includegraphics[width=\linewidth/8-0.5ex]{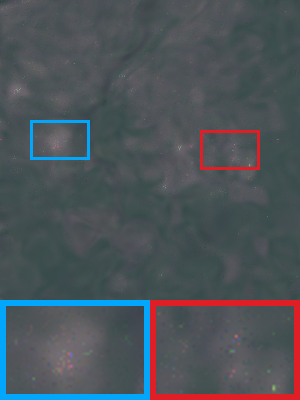} &
		\includegraphics[width=\linewidth/8-0.5ex]{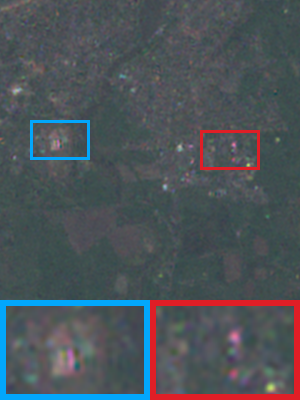} &
		\includegraphics[width=\linewidth/8-0.5ex]{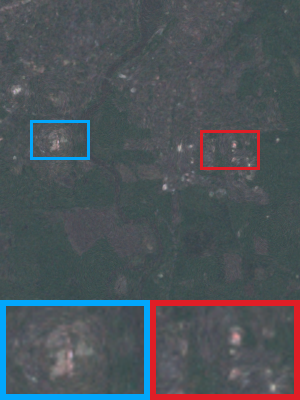} &
		\includegraphics[width=\linewidth/8-0.5ex]{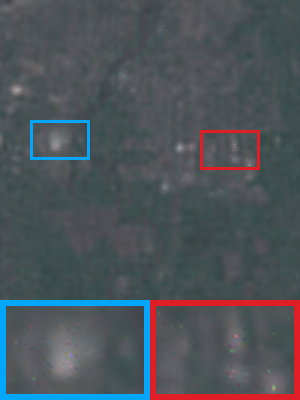} &
		\includegraphics[width=\linewidth/8-0.5ex]{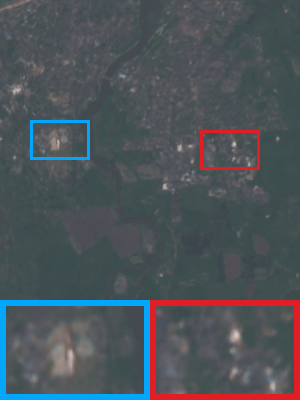} &
		\includegraphics[width=\linewidth/8-0.5ex]{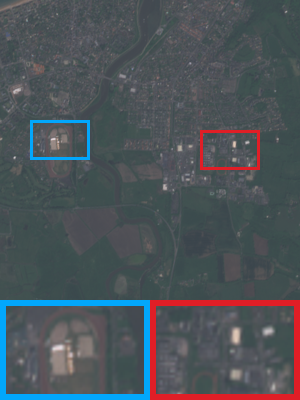} \\
		\includegraphics[width=\linewidth/8-0.5ex]{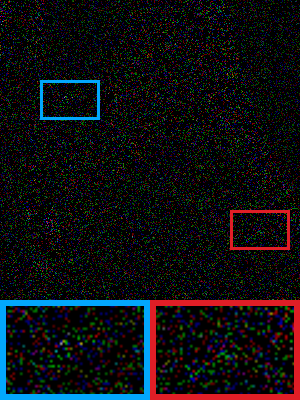} &
		\includegraphics[width=\linewidth/8-0.5ex]{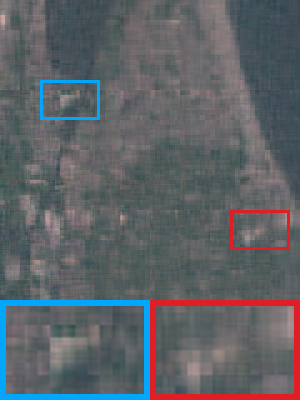} &
		\includegraphics[width=\linewidth/8-0.5ex]{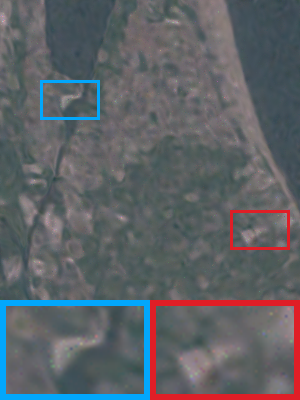} &
		\includegraphics[width=\linewidth/8-0.5ex]{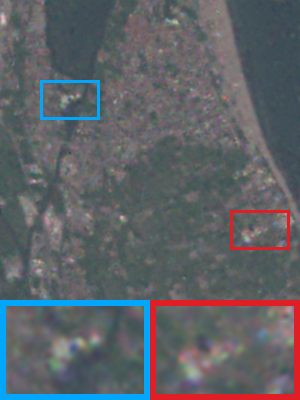} &
		\includegraphics[width=\linewidth/8-0.5ex]{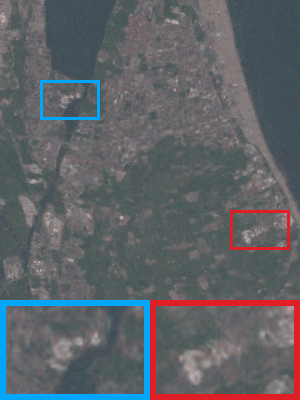} &
		\includegraphics[width=\linewidth/8-0.5ex]{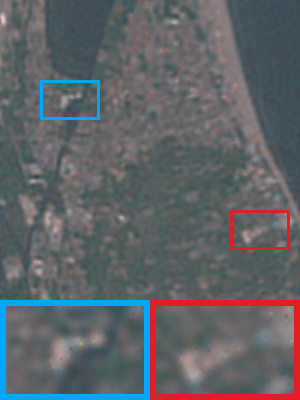} &
		\includegraphics[width=\linewidth/8-0.5ex]{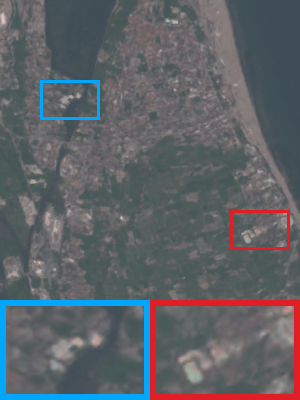} &
		\includegraphics[width=\linewidth/8-0.5ex]{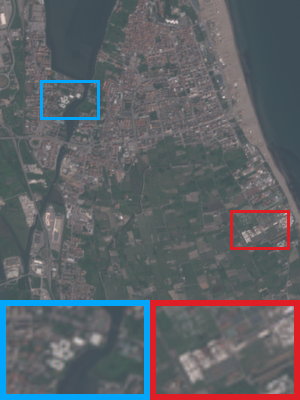} \\
		\includegraphics[width=\linewidth/8-0.5ex]{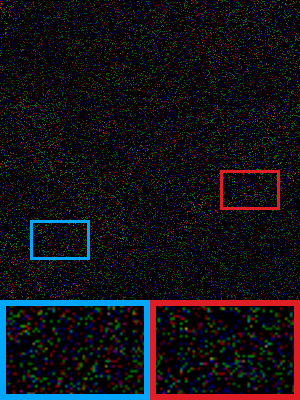} &
		\includegraphics[width=\linewidth/8-0.5ex]{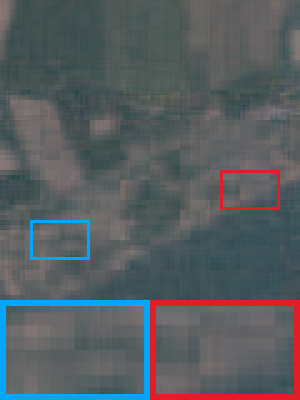} &
		\includegraphics[width=\linewidth/8-0.5ex]{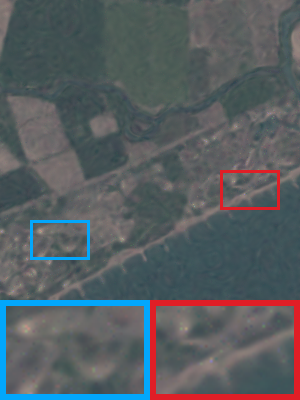} &
		\includegraphics[width=\linewidth/8-0.5ex]{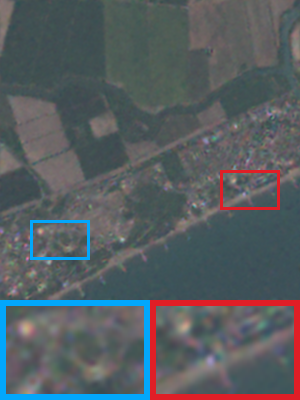} &
		\includegraphics[width=\linewidth/8-0.5ex]{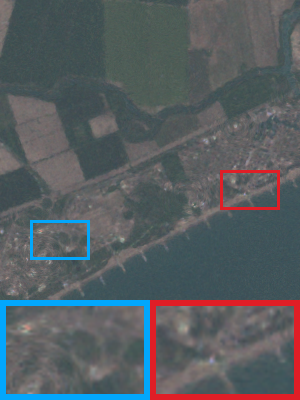} &
		\includegraphics[width=\linewidth/8-0.5ex]{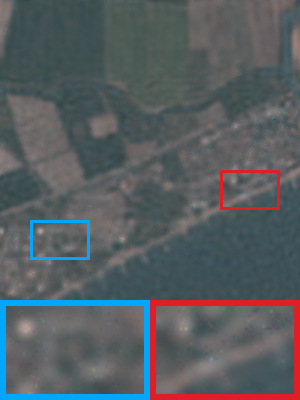} &
		\includegraphics[width=\linewidth/8-0.5ex]{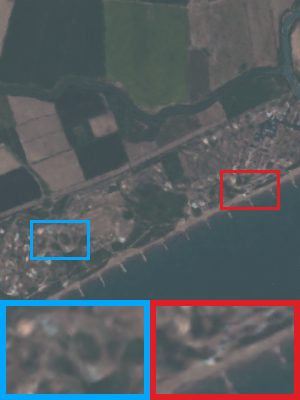} &
		\includegraphics[width=\linewidth/8-0.5ex]{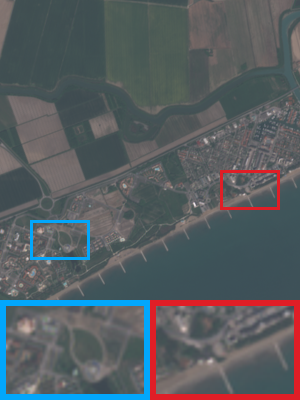}
	\end{tabular*}
	\caption{Visual recovery results of Sentinel-2 image completion at a sampling rate of $10\%$.} \label{fig:sentinel-2}
\end{figure*}

Next, we conduct experiments on multi-dimensional data on mesh grids with different spatial resolutions: Sentinel-2 images.
Sentinel-2 is an European mission that utilizes wide-swath, high-resolution, multi-spectral imaging.
It is equipped with an optical instrument payload that samples $13$ spectral bands, including: four bands at $10\,\mathrm{m}$, six bands at $20\,\mathrm{m}$, and three bands at $60\,\mathrm{m}$ spatial resolution. The corresponding images are cropped to $300 \times 300$ pixels, $150 \times 150$ pixels, and $50 \times 50$ pixels, respectively.

Table \ref{tab:sentinel-2} lists the quantitative recovery results of Sentinel-2 image completion. Proposed NO-CTR method exhibits outstanding performance across all three datasets and all evaluated sampling rates, outperforming competing methods in nearly all experimental settings. With a minor exception---at $5\%$ sampling rate on \textit{T30UYV}, FR-INR achieves a marginally higher PSNR. Notably, this superiority becomes more pronounced as the sampling rates increase, demonstrating its strong ability for Sentinel-2 image completion.
Figure \ref{fig:sentinel-2} shows the visual recovery results of Sentinel-2 image completion when the sampling rates are $10\%$. Proposed NO-CTR exhibits superior capability in recovering geographical details. When reconstructing satellite images, proposed NO-CTR recovers more accurate textures and boundaries for urban structures and land cover regions, which is vital for remote sensing analyses.

\subsection{Point Clouds}

\begin{table*}[!tp]
	\caption{Quantitative Recovery Results of Point Cloud Completion} \label{tab:pointcloud}
	\begin{tabular*}{\linewidth}{CCCCCCCCCC@{}}
		\toprule
		\multicolumn{2}{c}{Sampling Rate} & \multicolumn{2}{c}{5\%} & \multicolumn{2}{c}{10\%} & \multicolumn{2}{c}{15\%} & \multicolumn{2}{c}{20\%} \\
		\cmidrule{1-2} \cmidrule{3-4} \cmidrule{5-6} \cmidrule{7-8} \cmidrule{9-10}
		Point Cloud & Method & NRMSE & $R^2$ & NRMSE & $R^2$ & NRMSE & $R^2$ & NRMSE & $R^2$ \\
		\midrule
		\multirow{5}{*}{\includegraphics[height=6em]{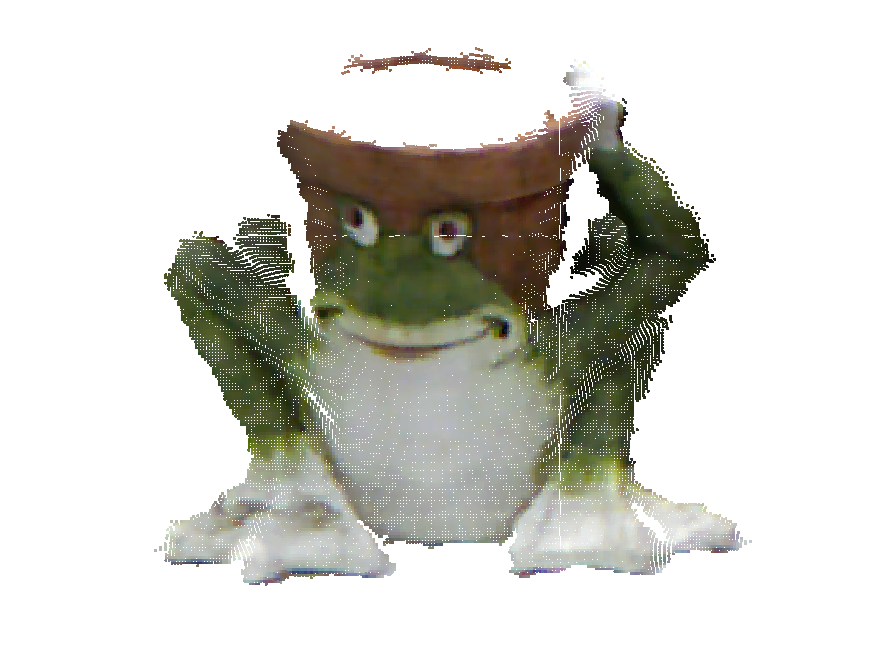}}
		& Observation & 0.470 & -3.120 & 0.457 & -2.894 & 0.443 & -2.662 & 0.432 & -2.484 \\
		& SIREN & \underline{0.054} & \underline{0.947} & \underline{0.044} & \underline{0.964} & \underline{0.039} & \underline{0.972} & \underline{0.031} & \underline{0.983} \\
		& MFN & 0.060 & 0.936 & 0.051 & 0.953 & 0.047 & 0.959 & 0.038 & 0.974 \\
		& FR-INR & 0.057 & 0.940 & 0.047 & 0.959 & 0.043 & 0.966 & 0.033 & 0.980 \\
		& LRTFR & 0.062 & 0.930 & 0.054 & 0.947 & 0.050 & 0.955 & 0.040 & 0.971 \\
		\textit{Frog} & NO-CTR & \textbf{0.051} & \textbf{0.953} & \textbf{0.041} & \textbf{0.969} & \textbf{0.034} & \textbf{0.980} & \textbf{0.028} & \textbf{0.985} \\
		\midrule
		\multirow{5}{*}{\includegraphics[height=6em]{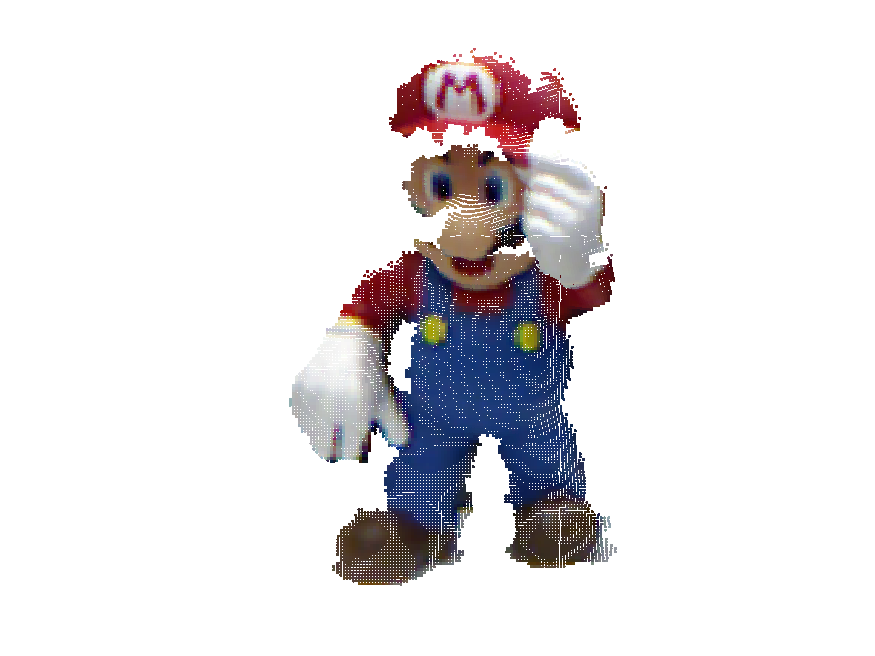}}
		& Observation & 0.435 & -1.939 & 0.423 & -1.788 & 0.412 & -1.637 & 0.398 & -1.469 \\
		& SIREN & \underline{0.088} & \underline{0.878} & \underline{0.063} & \underline{0.938} & \underline{0.058} & \underline{0.948} & \underline{0.044} & \underline{0.970} \\
		& MFN & 0.096 & 0.855 & 0.069 & 0.925 & 0.067 & 0.931 & 0.052 & 0.957 \\
		& FR-INR & 0.091 & 0.871 & 0.064 & 0.936 & 0.062 & 0.940 & 0.045 & 0.968 \\
		& LRTFR & 0.099 & 0.846 & 0.077 & 0.906 & 0.074 & 0.915 & 0.056 & 0.951 \\
		\textit{Mario} & NO-CTR & \textbf{0.083} & \textbf{0.892} & \textbf{0.061} & \textbf{0.942} & \textbf{0.055} & \textbf{0.952} & \textbf{0.043} & \textbf{0.971} \\
		\midrule
		\multirow{5}{*}{\includegraphics[height=6em]{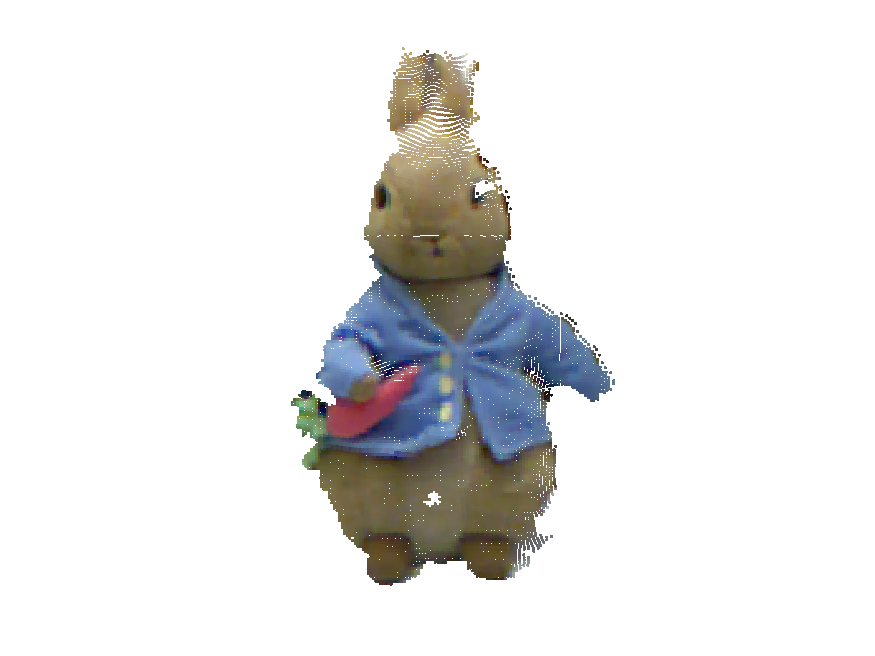}}
		& Observation & 0.465 & -8.562 & 0.451 & -8.003 & 0.439 & -7.546 & 0.427 & -7.077 \\
		& SIREN & \underline{0.082} & \underline{0.709} & \underline{0.069} & \underline{0.795} & \underline{0.056} & \underline{0.863} & \underline{0.045} & \underline{0.914} \\
		& MFN & 0.088 & 0.669 & 0.071 & 0.781 & 0.060 & 0.840 & 0.048 & 0.903 \\
		& FR-INR & 0.095 & 0.604 & 0.070 & 0.783 & 0.058 & 0.851 & 0.046 & 0.908 \\
		& LRTFR & 0.088 & 0.671 & 0.075 & 0.756 & 0.064 & 0.818 & 0.054 & 0.874 \\
		\textit{Rabbit} & NO-CTR & \textbf{0.077} & \textbf{0.740} & \textbf{0.062} & \textbf{0.831} & \textbf{0.047} & \textbf{0.906} & \textbf{0.040} & \textbf{0.930} \\
		\bottomrule
	\end{tabular*}
	The \textbf{best} and \underline{second-best} results are highlighted.
\end{table*}

\begin{figure*}[!tp]
	\footnotesize
	\begin{tabular*}{\linewidth}{CCCCCCC@{}}
		Observation & SIREN & MFN & FR-INR & LRTFR & NO-CTR & Ground Truth \\
		\includegraphics[width=\linewidth/7-0.5ex]{./pic/frog_obs.png} &
		\includegraphics[width=\linewidth/7-0.5ex]{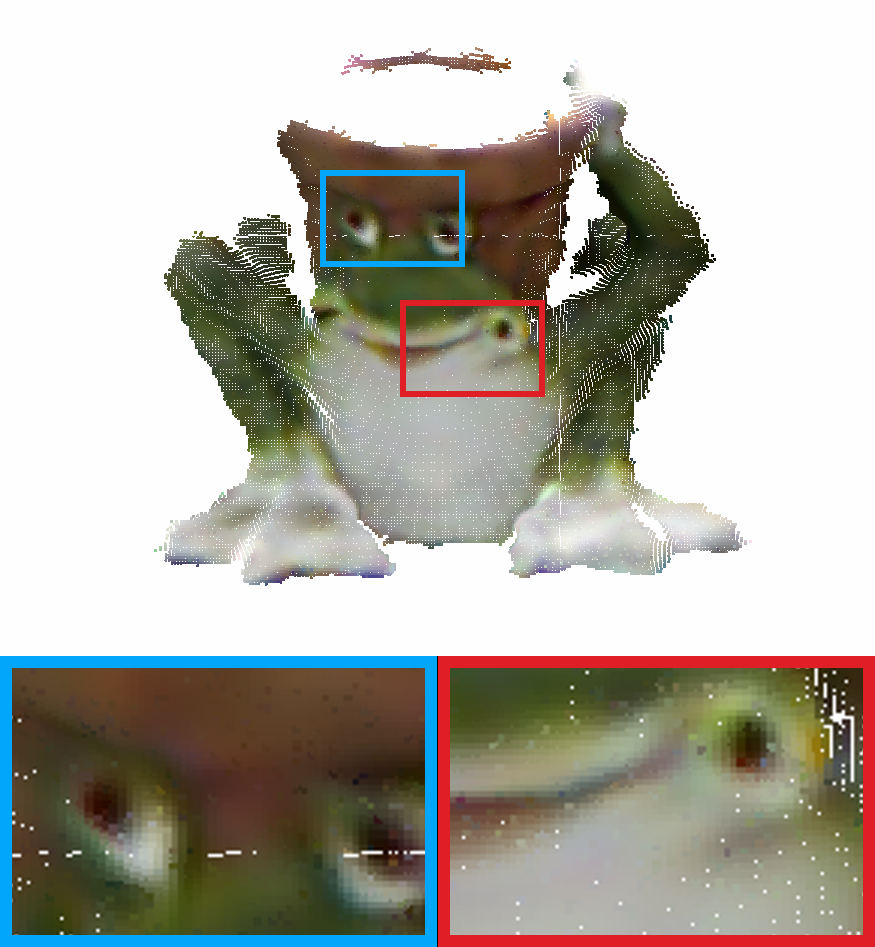} &
		\includegraphics[width=\linewidth/7-0.5ex]{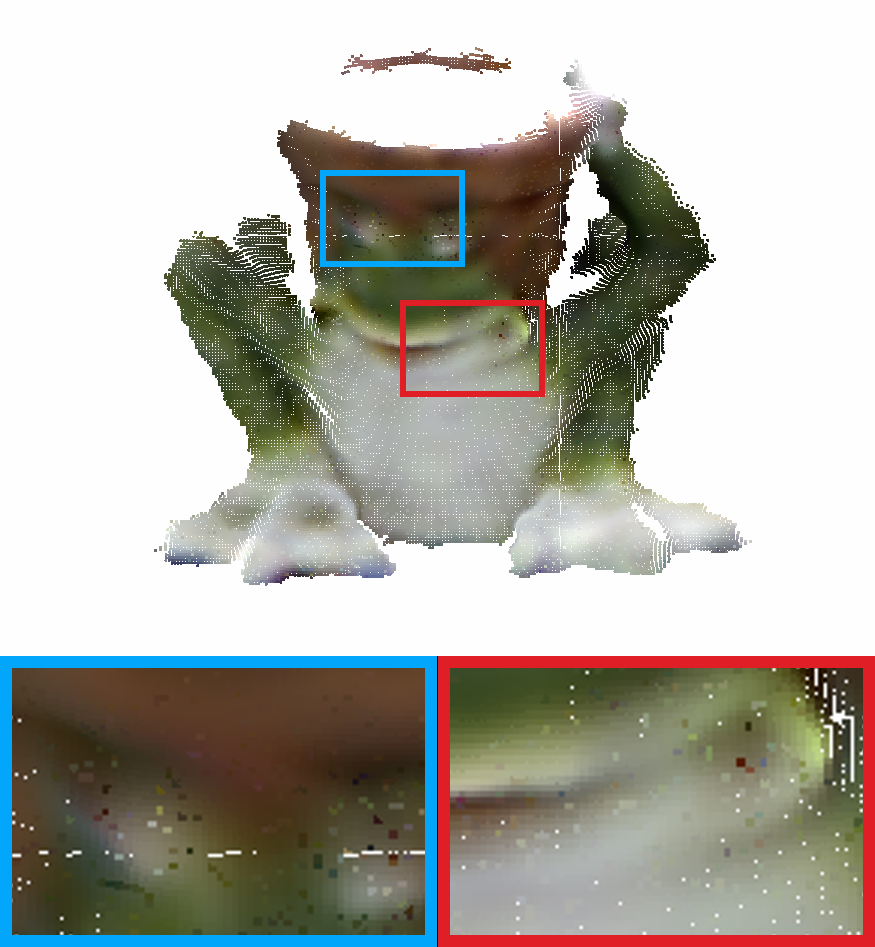} &
		\includegraphics[width=\linewidth/7-0.5ex]{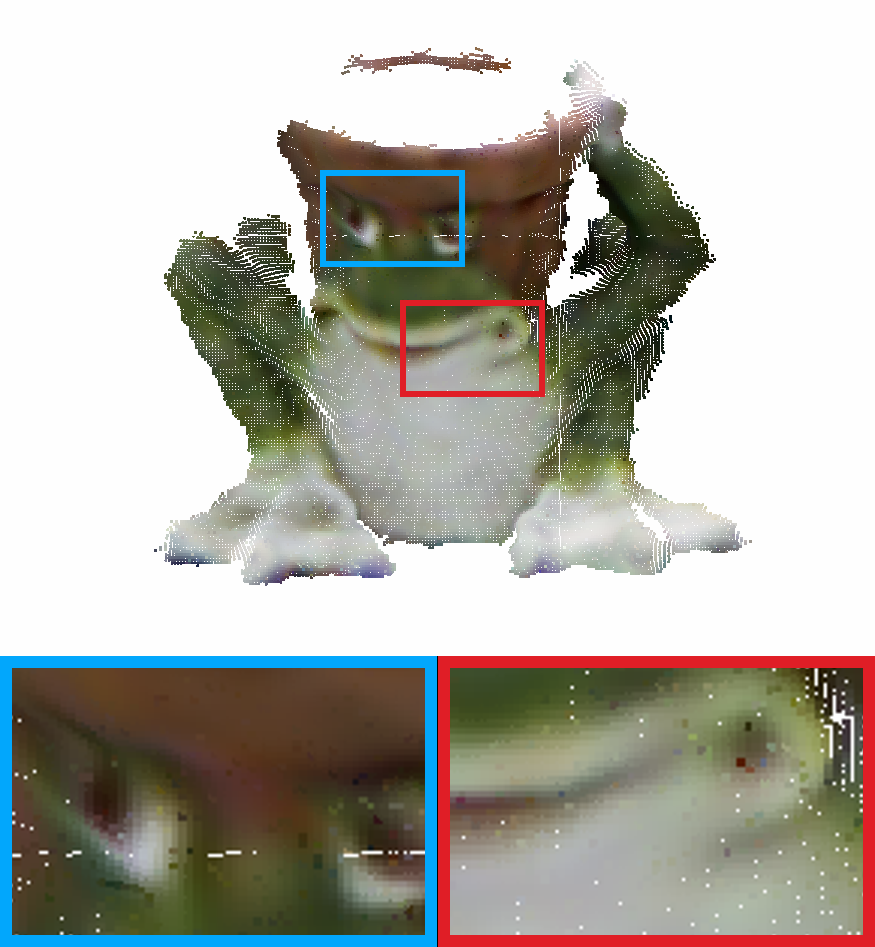} &
		\includegraphics[width=\linewidth/7-0.5ex]{./pic/frog_lrtfr.png} &
		\includegraphics[width=\linewidth/7-0.5ex]{./pic/frog_ours.png} &
		\includegraphics[width=\linewidth/7-0.5ex]{./pic/frog_gt.png} \\
		\includegraphics[width=\linewidth/7-0.5ex]{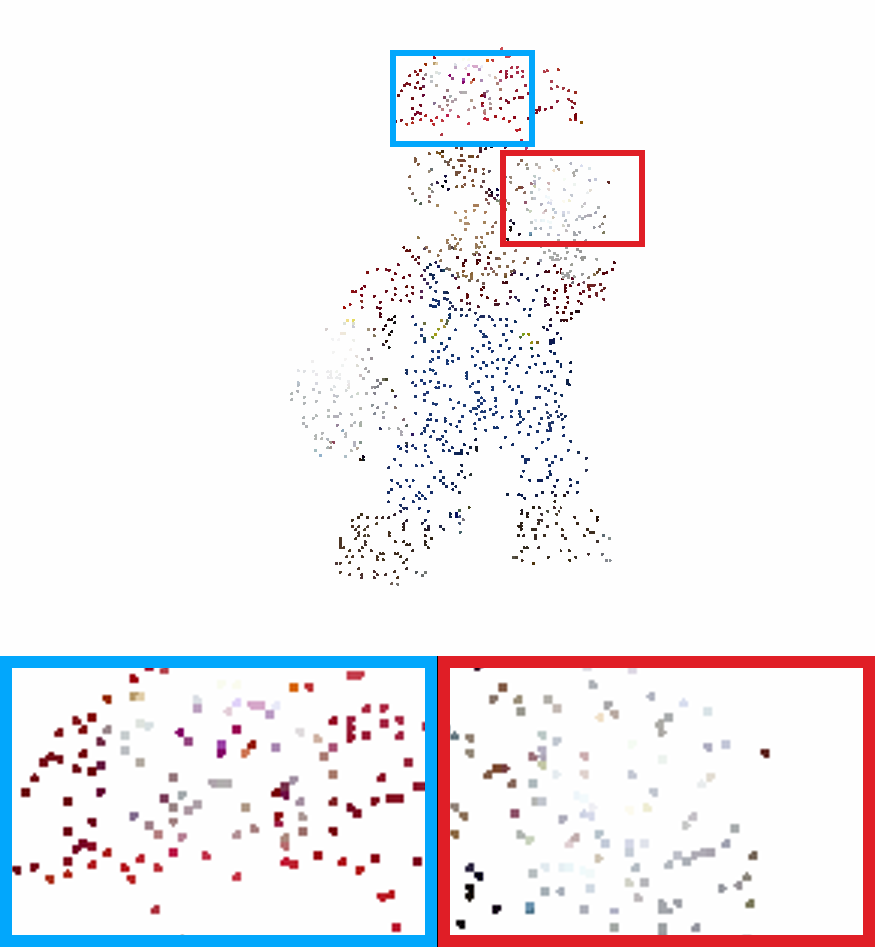} &
		\includegraphics[width=\linewidth/7-0.5ex]{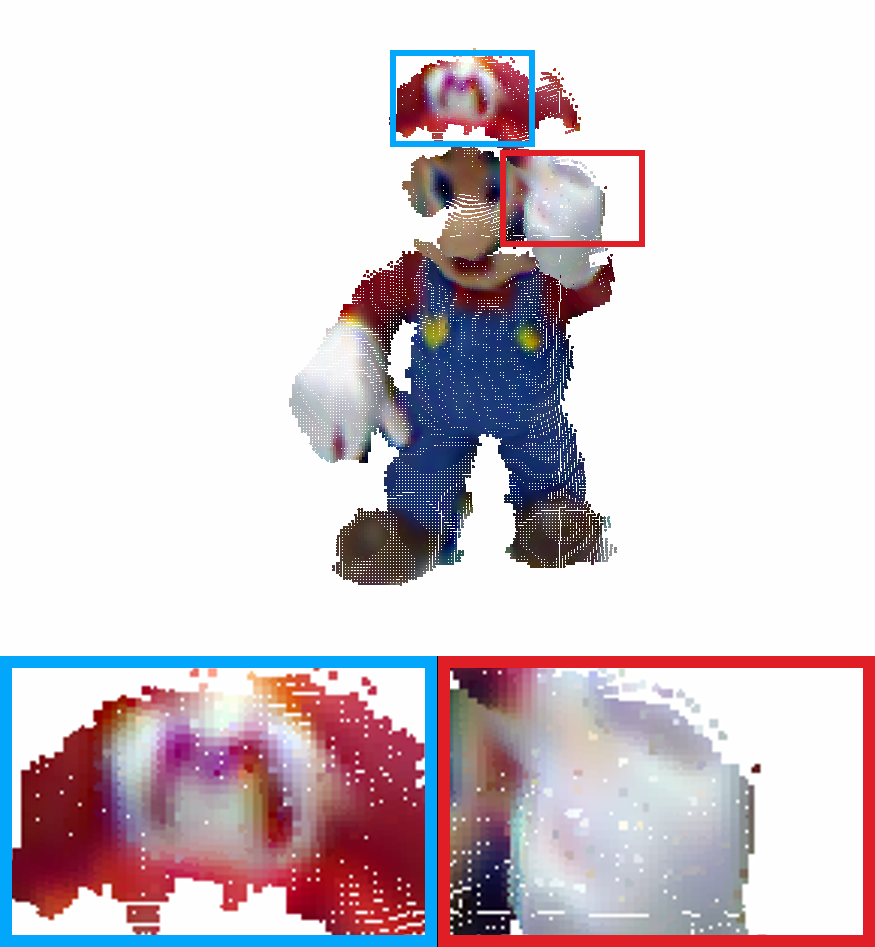} &
		\includegraphics[width=\linewidth/7-0.5ex]{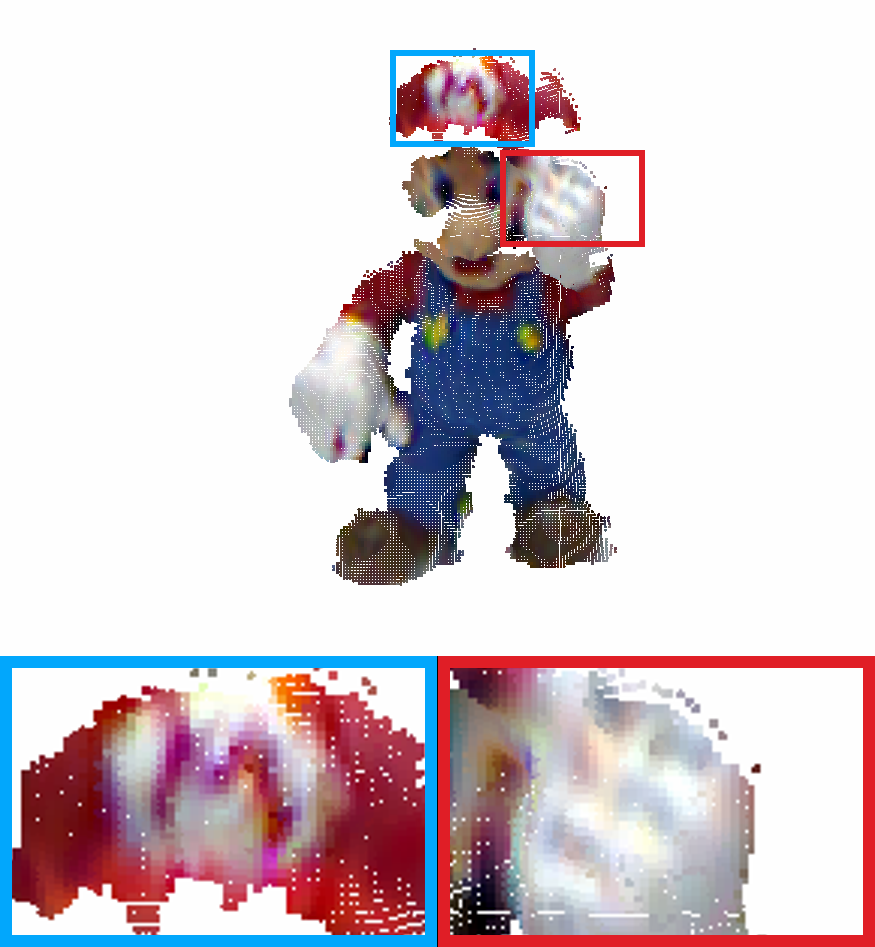} &
		\includegraphics[width=\linewidth/7-0.5ex]{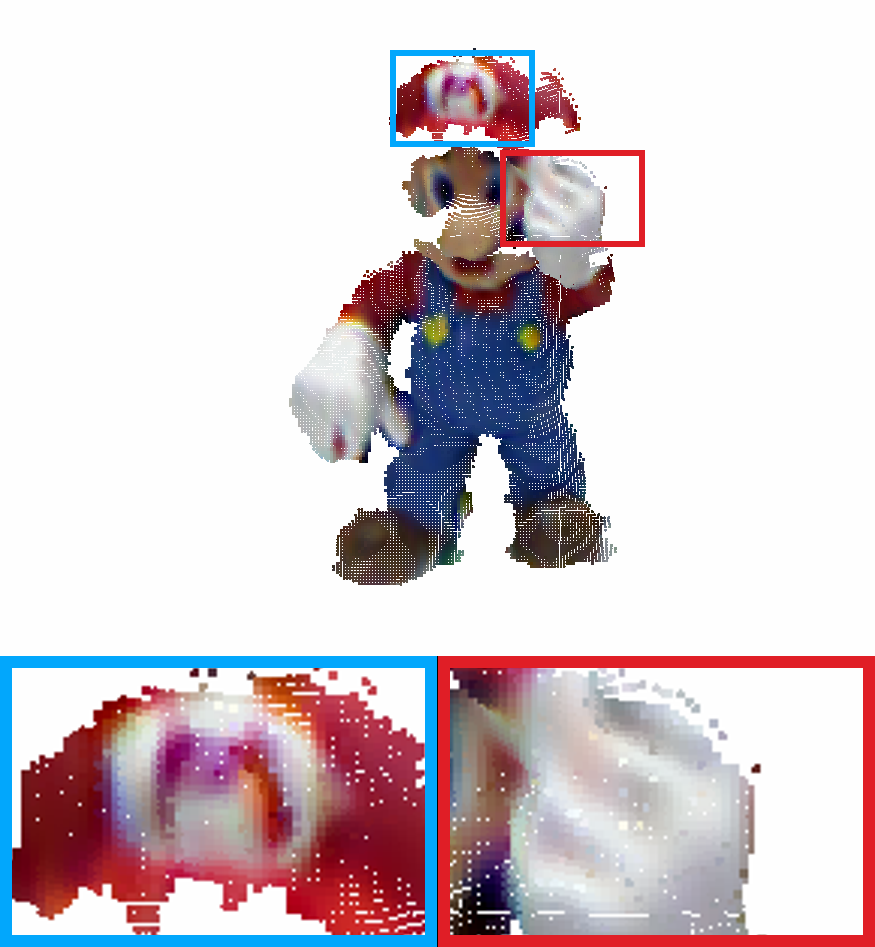} &
		\includegraphics[width=\linewidth/7-0.5ex]{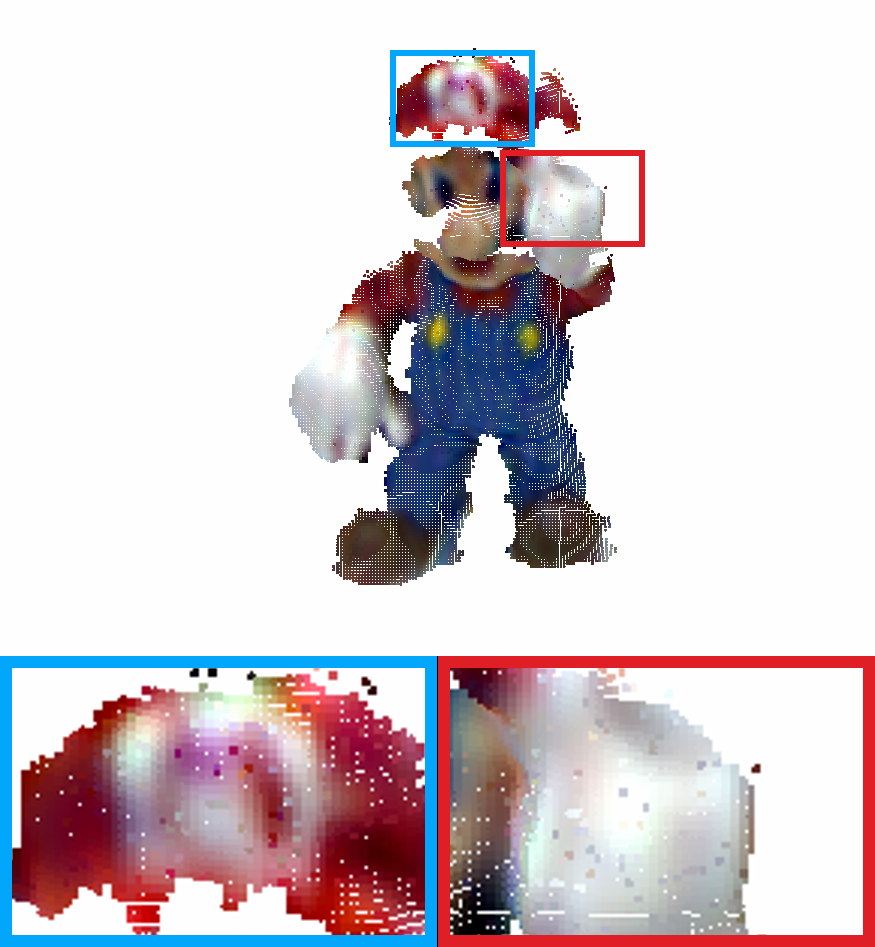} &
		\includegraphics[width=\linewidth/7-0.5ex]{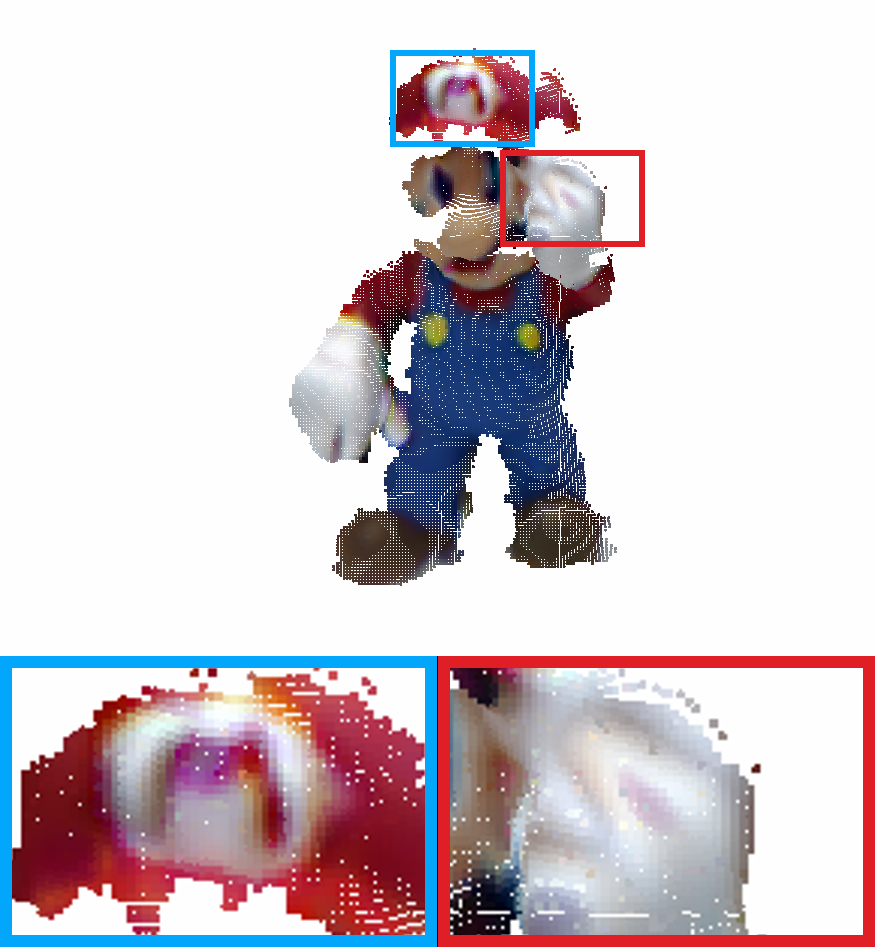} &
		\includegraphics[width=\linewidth/7-0.5ex]{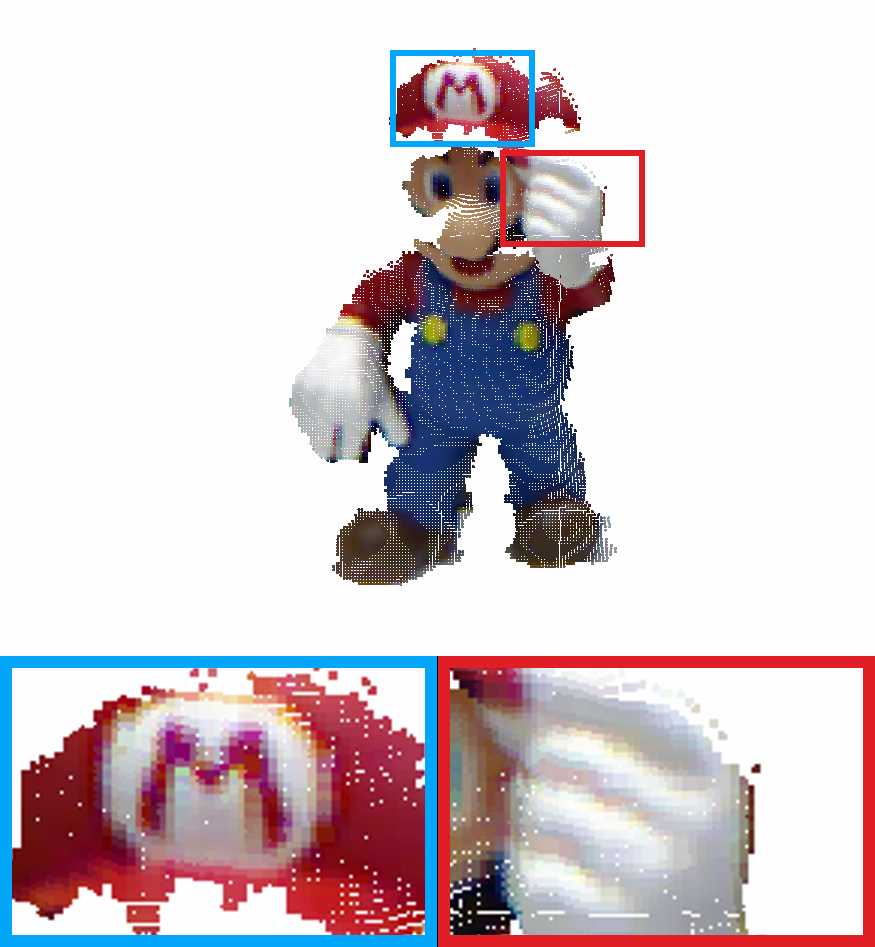} \\
		\includegraphics[width=\linewidth/7-0.5ex]{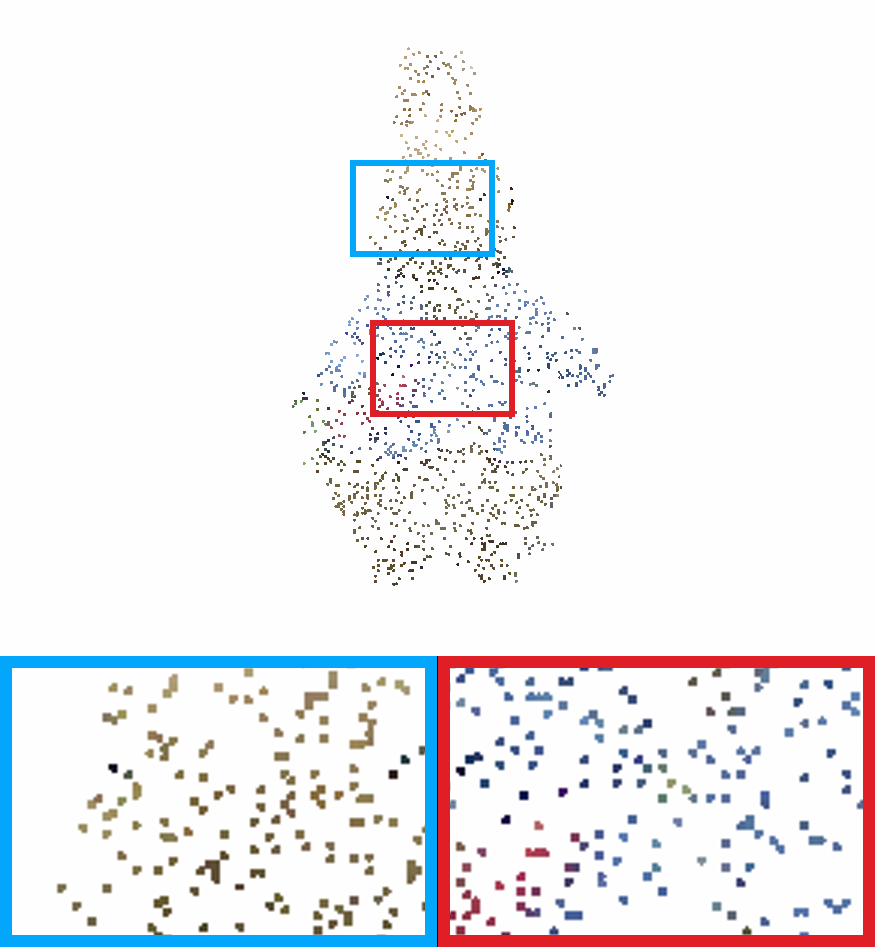} &
		\includegraphics[width=\linewidth/7-0.5ex]{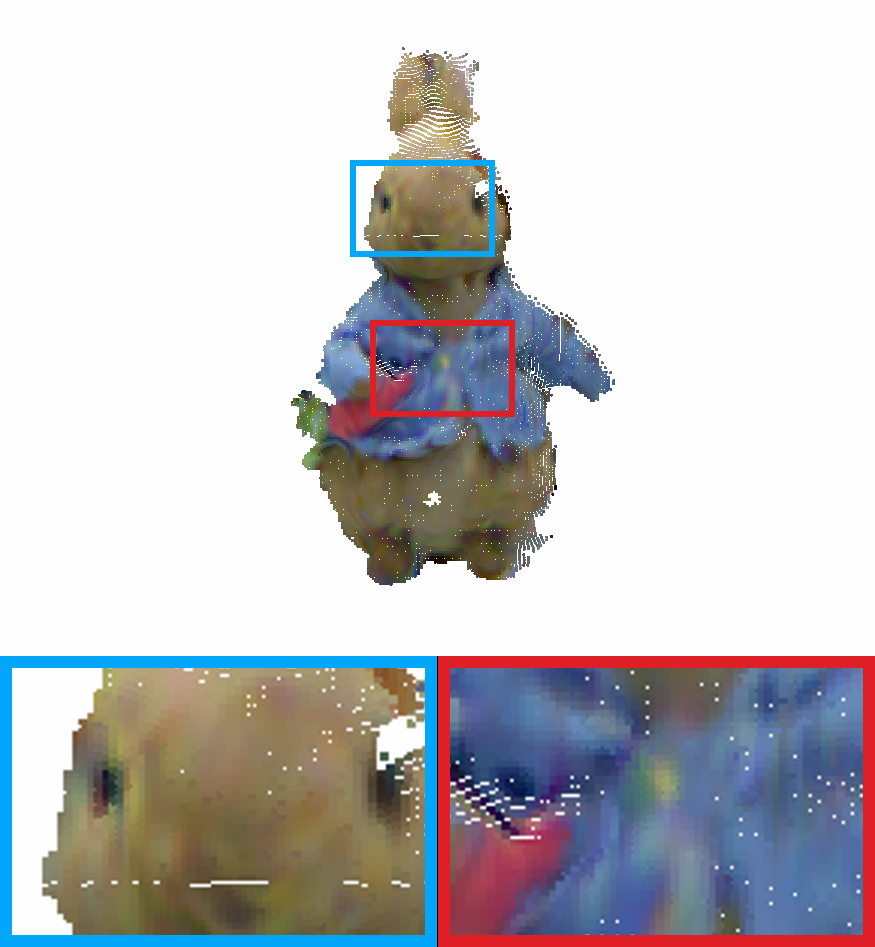} &
		\includegraphics[width=\linewidth/7-0.5ex]{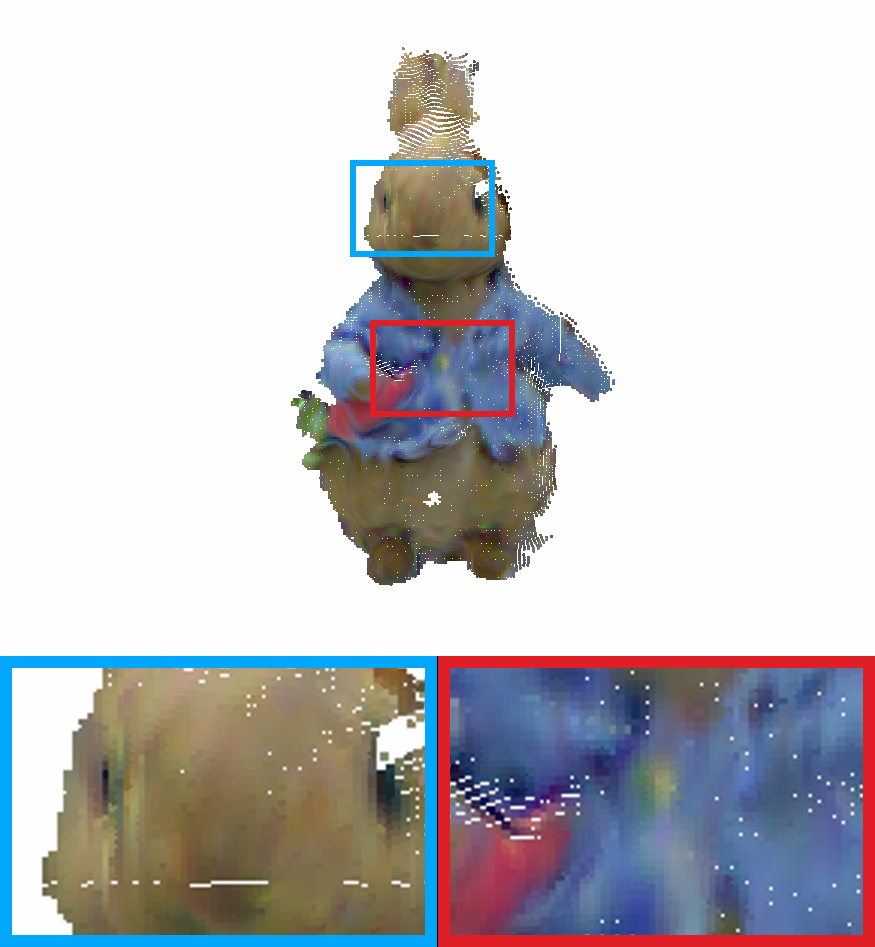} &
		\includegraphics[width=\linewidth/7-0.5ex]{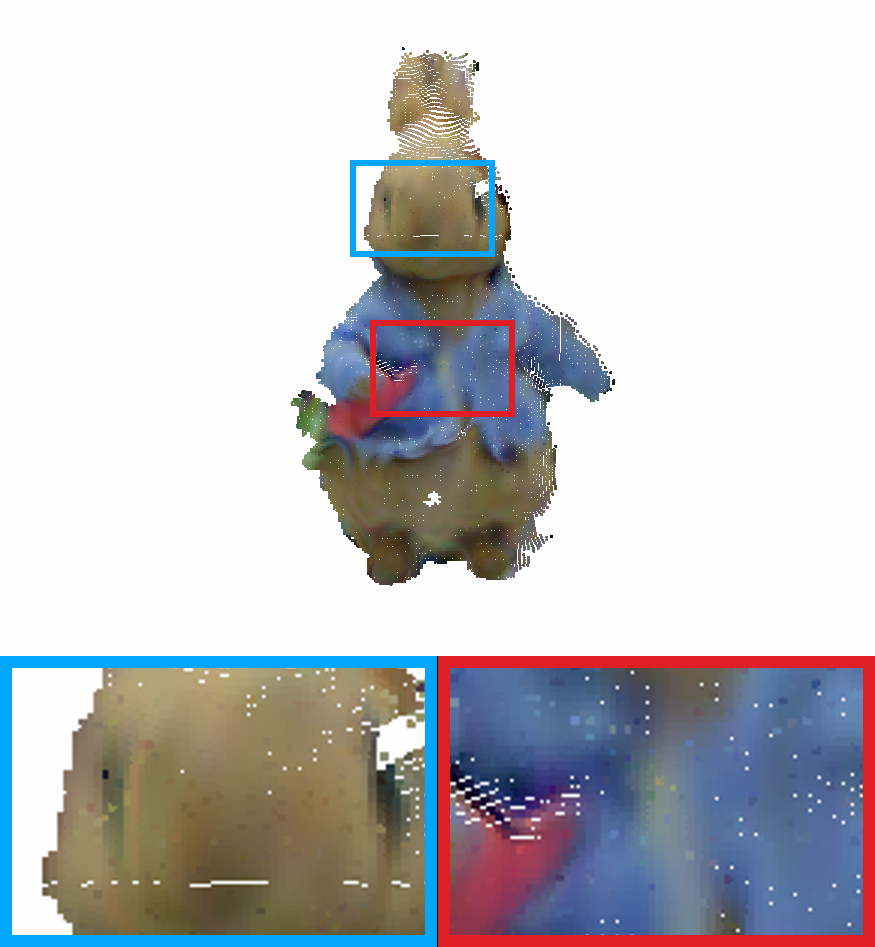} &
		\includegraphics[width=\linewidth/7-0.5ex]{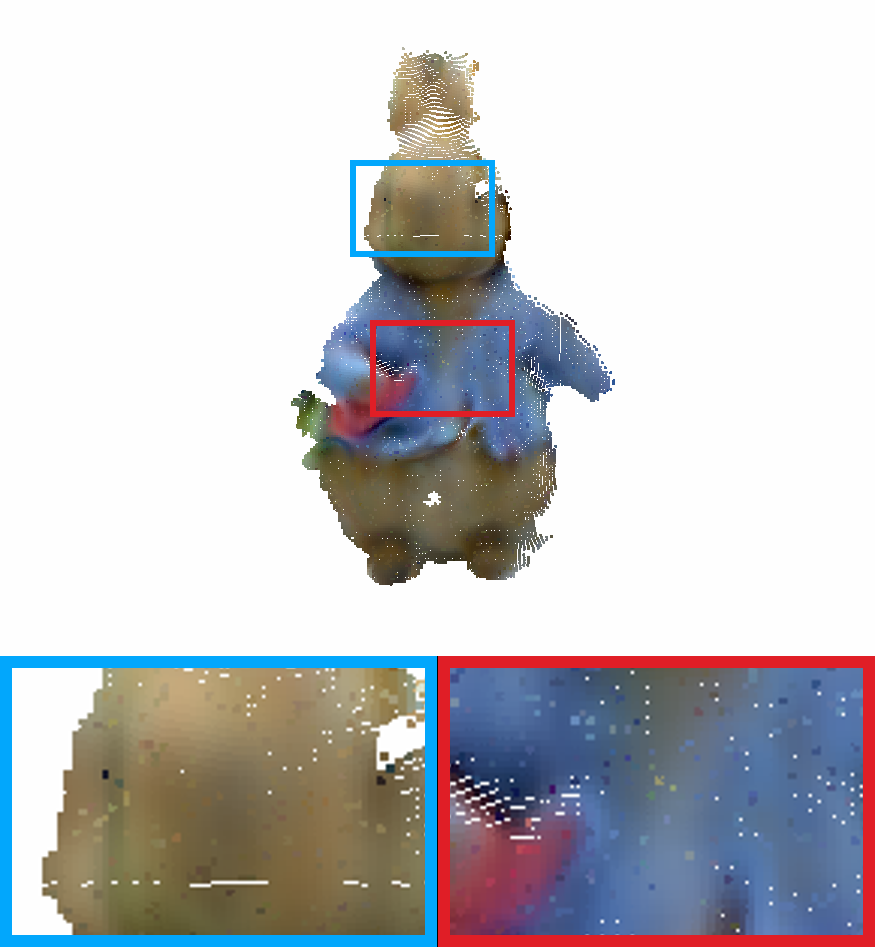} &
		\includegraphics[width=\linewidth/7-0.5ex]{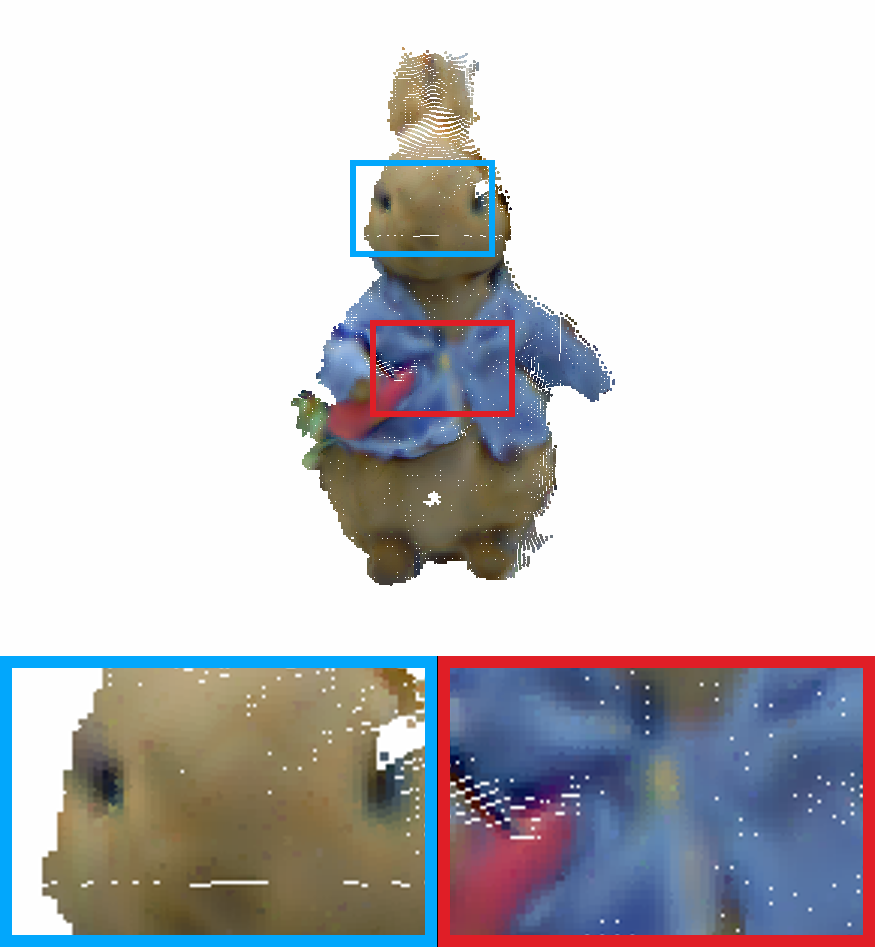} &
		\includegraphics[width=\linewidth/7-0.5ex]{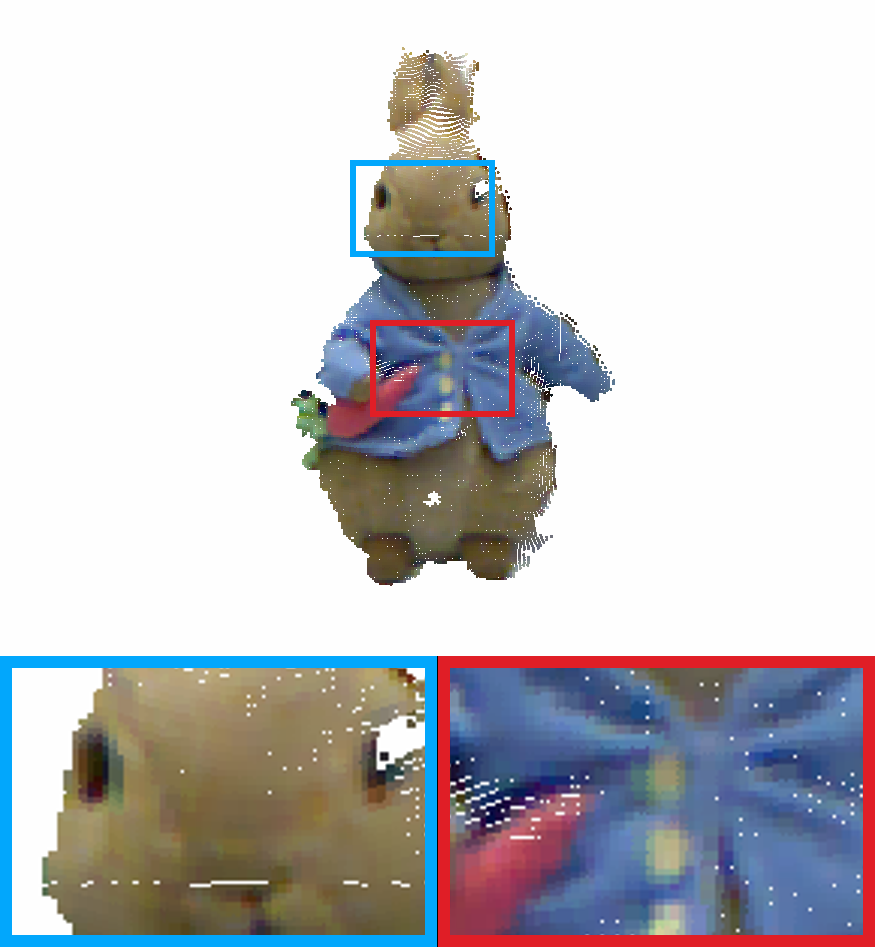}

	\end{tabular*}
	\caption{Visual recovery results of point cloud completion at a sampling rate of $10\%$.} \label{fig:point cloud}
\end{figure*}

Finally, we conduct experiments on multi-dimensional data beyond mesh grids: point clouds, which can hardly be achieved by traditional tensor completion models.
Point clouds are widely used in various fields, including computer-aided design, virtual reality, and autonomous driving, to represent the shape or surface characteristics of objects or environments. Each point cloud refers to a collection of points in three-dimensional space, where each point contains spatial coordinates and RGB values.
In our experiments, each point cloud contains about $20\,000$ points. We random select $5\%$, $10\%$, $15\%$, or $20\%$ of the points as the observation respectively, and recover the RGB values at the remaining points based on their coordinates.

Table \ref{tab:pointcloud} lists the quantitative recovery results of point cloud completion. Proposed NO-CTR method consistently outperforms all competing methods across all point clouds and all sampling rates, as indicated by the lowest NRMSE and highest $R^2$ in every experimental setting. Notably, this superiority holds even at both low and high sampling rates, demonstrating its effectiveness for point cloud completion.
Figure \ref{fig:point cloud} shows the visual recovery results of point cloud completion when the sampling rates are $10\%$. Proposed NO-CTR successfully reconstructs the surface details of these 3D models. The point clouds recovered by proposed NO-CTR contain more accurate details, especially in capturing distinctive features.

\section{Discussion}

\subsection{Contribution of Continuous and Nonlinear Mode-$n$ Operators}

The continuous and nonlinear mode-$n$ operators are the heart of the NO-CTR. To discuss the contribution of continuous and nonlinear mode-$n$ operators, in this subsection, we consider different operators, including discrete and linear mode-$n$ operators, and continuous and nonlinear mode-$n$ operators.
Table \ref{tab:mode-n operators} lists the quantitative recovery results of MSI completion with different operators.
It is obvious that continuous and nonlinear mode-$n$ operators play a huge role in improving the representation capabilities of NO-CTR, achieving more powerful representation capabilities. Compared with the previous methods with discrete and linear mode-$n$ operators, proposed NO-CTR with continuous and nonlinear mode-$n$ operators demonstrates significant advantages.

\begin{table}[!htb]
	\caption{Quantitative Recovery Results of MSI Completion With Different Mode-$n$ Operators} \label{tab:mode-n operators}
	\begin{tabular*}{\linewidth}{CCCCCC@{}}
		\toprule
		\multicolumn{2}{c}{Sampling Rate} & \multicolumn{2}{c}{10\%} & \multicolumn{2}{c}{20\%} \\
		\cmidrule{1-2} \cmidrule{3-4} \cmidrule{5-6}
		MSI & Mode-$n$ Operators & PSNR & SSIM & PSNR & SSIM \\
		\midrule
		\multirow{3}{*}{\emph{Flowers}}
		& No Operator & 30.565 & 0.898 & 34.367 & 0.948 \\
		& Discrete \& Linear & 39.175 & 0.984 & 43.119 & 0.984 \\
		& Continuous \& Nonlinear & \textbf{42.591} & \textbf{0.994} & \textbf{49.717} & \textbf{0.999} \\
		\midrule
		\multirow{3}{*}{\emph{Toy}}
		& No Operator & 26.190 & 0.835 & 29.087 & 0.895 \\
		& Discrete \& Linear & 38.907 & 0.982 & 41.395 & 0.991 \\
		& Continuous \& Nonlinear & \textbf{42.294} & \textbf{0.995} & \textbf{48.878} & \textbf{0.999} \\
		\bottomrule
	\end{tabular*}
	The \textbf{best} results are highlighted.
\end{table}

\subsection{Architectures of Neural Operators} \label{sec:contribution of neural operators}

In NO-CTR, the continuous and nonlinear mode-$n$ operators can be elegantly deployed with neural operators, which directly map the continuous core tensor function to the continuous target tensor function.
To discuss the influence of neural operators, in this subsection, we consider two popular neural operators: Fourier neural operators (FNO) \cite{li2021fourier} and deep operator networks (DeepONet) \cite{lu2021learning}.
Table \ref{tab:inducing operators} lists the quantitative recovery results of MSI completion with different neural operators. Neural operators indeed enhance the performance of the NO-CTR, and DeepONet achieves better performance due to its flexible architecture.

\begin{table}[!htb]
	\caption{Quantitative Recovery Results of MSI Completion With Different Neural Operators} \label{tab:inducing operators}
	\begin{tabular*}{\linewidth}{CCCCCC@{}}
		\toprule
		\multicolumn{2}{c}{Sampling Rate} & \multicolumn{2}{c}{10\%} & \multicolumn{2}{c}{20\%} \\
		\cmidrule{1-2} \cmidrule{3-4} \cmidrule{5-6}
		MSI & Architecture & PSNR & SSIM & PSNR & SSIM \\
		\midrule
		\multirow{3}{*}{\emph{Flowers}}
		& No Operator & 30.565 & 0.898 & 34.367 & 0.948 \\
		& FNO & 40.940 & 0.992 & 47.446 & 0.997 \\
		& DeepONet & \textbf{42.591} & \textbf{0.994} & \textbf{49.717} & \textbf{0.999} \\
		\midrule
		\multirow{3}{*}{\emph{Toy}}
		& No Operator & 26.190 & 0.835 & 29.087 & 0.895 \\
		& FNO & 41.167 & 0.994 & 46.076 & 0.998 \\
		& DeepONet & \textbf{42.294} & \textbf{0.995} & \textbf{48.878} & \textbf{0.999} \\
		\bottomrule
	\end{tabular*}
	The \textbf{best} results are highlighted.
\end{table}

\subsection{Influence of Continuous Core Tensor Function} \label{sec:contribution of continuous core tensor function}

The continuous core tensor function is also an important element of the NO-CTR. In this subsection, we discuss the influence of the continuous core tensor function.
Here we consider different representation methods for the continuous core tensor function, including LRTFR \cite{luo2024low}, positional encoding-based multi-layer perceptron (referred to as PE-MLP) \cite{tancik2020fourier}, and SIREN \cite{sitzmann2020implicit}.
Table \ref{tab:core tensors} list the quantitative recovery results of MSI completion with different representation methods for the continuous core tensor function.
LRTFR imposes low-rank constraints on the continuous core tensor functions, which limits its representation capability to some extent. SIREN demonstrates the best performance here.

\begin{table}[!htb]
	\caption{Quantitative Recovery Results of MSI Completion With Different Representation Methods for the Continuous Core Tensor Function} \label{tab:core tensors}
	\begin{tabular*}{\linewidth}{CCCCCC@{}}
		\toprule
		\multicolumn{2}{c}{Sampling Rate} & \multicolumn{2}{c}{10\%} & \multicolumn{2}{c}{20\%} \\
		\cmidrule{1-2} \cmidrule{3-4} \cmidrule{5-6}
		MSI & Core & PSNR & SSIM & PSNR & SSIM \\
		\midrule
		\multirow{3}{*}{\emph{Flowers}}
		& LRTFR & 39.586 & 0.983 & 44.232 & 0.992 \\
		& PE-MLP & 42.086 & 0.993 & 47.326 & 0.998 \\
		& SIREN & \textbf{42.591} & \textbf{0.994} & \textbf{49.717} & \textbf{0.999} \\
		\midrule
		\multirow{3}{*}{\emph{Toy}}
		& LRTFR & 37.781 & 0.980 & 41.948 & 0.991 \\
		& PE-MLP & 40.831 & 0.993 & 45.765 & 0.998\\
		& SIREN & \textbf{42.294} & \textbf{0.995} & \textbf{48.878} & \textbf{0.999} \\
		\bottomrule
	\end{tabular*}
	The \textbf{best} results are highlighted.
\end{table}

\subsection{Number of Learnable Parameters}

In this subsection, we compare the number of learnable parameters among different continuous tensor function representations, including SIREN \cite{sitzmann2020implicit}, LRTFR \cite{luo2024low}, and proposed NO-CTR.
The number of learnable parameters is a key indicator of model complexity, and is also related to training costs and speed.
The proposed NO-CTR has the largest number of parameters, as it involves a continuous core tensor function and a series of neural operator-based continuous and nonlinear mode-$n$ operators.
Despite having more parameters, NO-CTR achieves significantly better recovery performance. This indicates that the increase in parameters is reasonable and effective.

\begin{table}[!htb]
	\caption{Quantitative Recovery Results and Number of Parameters of MSI Completion With Different Continuous Tensor Function Representations}
	\begin{tabular*}{\linewidth}{CCCCCCCC@{}}
		\toprule
		\multicolumn{2}{c}{Sampling Rate} & \multicolumn{3}{c}{10\%} & \multicolumn{3}{c}{20\%} \\
		\cmidrule{1-2} \cmidrule{3-5} \cmidrule{6-8}
		MSI & Method & PSNR & SSIM & Parameters & PSNR & SSIM & Parameters \\
		\midrule
		\multirow{3}{*}{\emph{Flowers}}
		& SIREN & 38.681 & 0.977 & \textbf{16\,711} & 46.041 & 0.995 & \textbf{47\,671} \\
		& LRTFR & 36.879 & 0.965 & 130\,191 & 43.249 & 0.991 & 331\,023  \\
		& NO-CTR & \textbf{42.591} & \textbf{0.994} & 507\,341 & \textbf{49.717} & \textbf{0.999} & 472\,721 \\
		\midrule
		\multirow{3}{*}{\emph{Toy}}
		& SIREN & 37.259 & 0.974 & \textbf{16\,711} & 45.306 & 0.995 & \textbf{47\,671} \\
		& LRTFR & 37.661 & 0.975 & 130\,191 & 43.099 & 0.991 & 595\,231 \\
		& NO-CTR & \textbf{42.294} & \textbf{0.995} & 597\,791 & \textbf{48.878} & \textbf{0.999} & 599\,291 \\
		\bottomrule
	\end{tabular*}
	The \textbf{best} results are highlighted.
\end{table}

\subsection{Hyper-Parameter Sensitivity}

In this subsection, we discuss the hyper-parameter sensitivity. We focus on the heart of the NO-CTR: continuous and nonlinear mode-$n$ operators.
The continuous and nonlinear mode-$n$ operators are elegantly deployed with DeepONets \cite{lu2021learning}. Specifically, for each inducing operator in Equation \eqref{eq:no-ctr},
\begin{equation} \label{eq:deeponet}
	F_n(v)(y) = \sum_{p=1}^{P_n} b^{(n)}_p\bigl(v(z_1),v(z_2),\cdots,v(z_{m_n})\bigr) t^{(n)}_p(y),
\end{equation}
where $v \in C([0,1])$ is the input function, $y \in [0,1]$ is a coordinate, $\{b^{(n)}_p\}_{p=1}^{P_n}$ are scalar-valued functions called branch networks, $\mat{t}^{(n)}$ is a vector-valued function called trunk network, $\{t^{(n)}_p\}_{p=1}^{P_n}$ are $P_n$ components of $\mat{t}^{(n)}$, and $z_i \in [0,1]$ ($i=1,2,\cdots,m_n$) are fixed location called sensors.
Here we consider two hyper-parameters in DeepONet: the number of sensors $m_n$ and the number of branch networks $P_n$.

\subsubsection{Number of Sensors}

DeepONet requires the input function values at a sufficient but finite number of locations (i.e. $\{z_i\}_{i=1}^{m_n}$ in Equation \eqref{eq:deeponet}) and these locations are called ``sensors''. The number of sensors (i.e. $m_n$ in Equation \eqref{eq:deeponet}) directly determines the accuracy and information of the input of DeepONet.
Figure \ref{fig:sensors} shows the quantitative recovery results w.r.t. different numbers of sensors. As the number of sensors increases, both PSNR and SSIM values first rise sharply; after that, they exhibit a stable trend with slight fluctuations, implying that a proper number of sensors is sufficient to achieve favorable recovery performance.

\begin{figure}[!htb]
	\centering
	\includegraphics[width=0.49\linewidth]{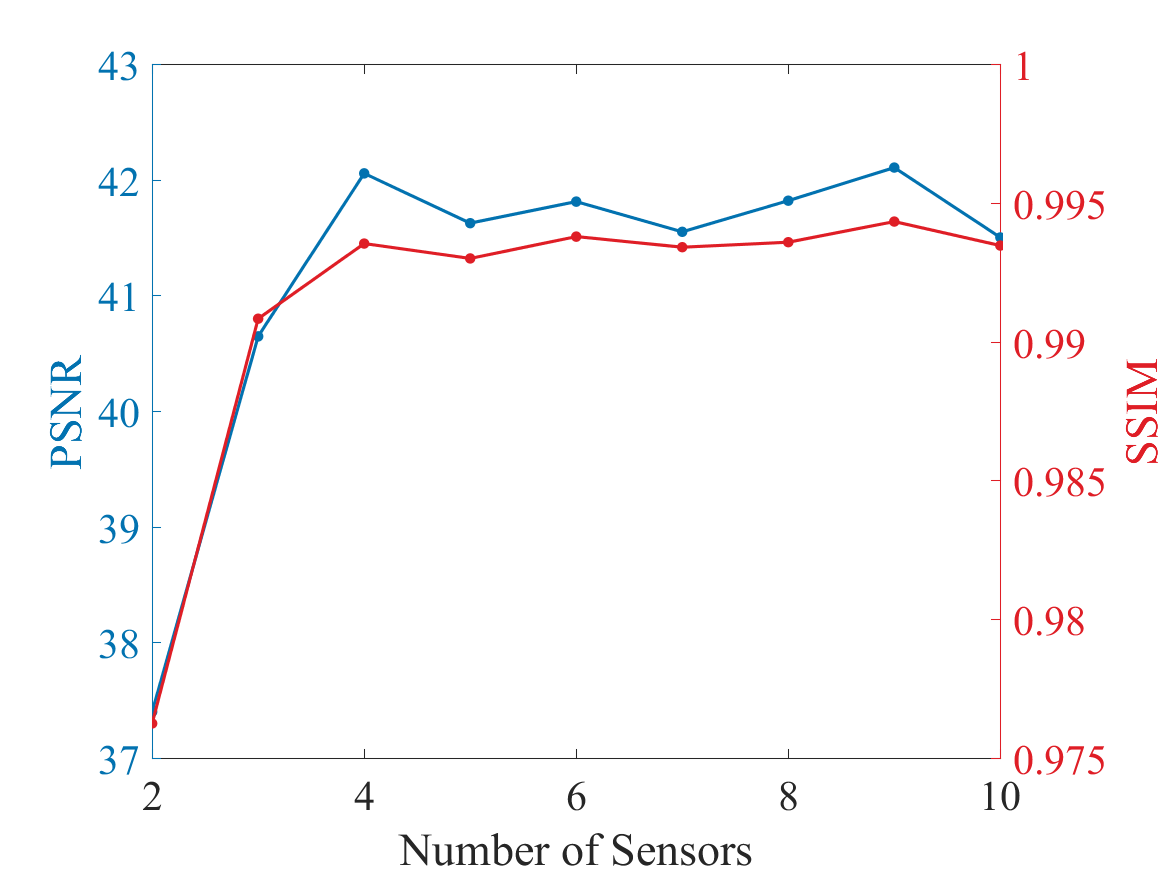}
	\includegraphics[width=0.49\linewidth]{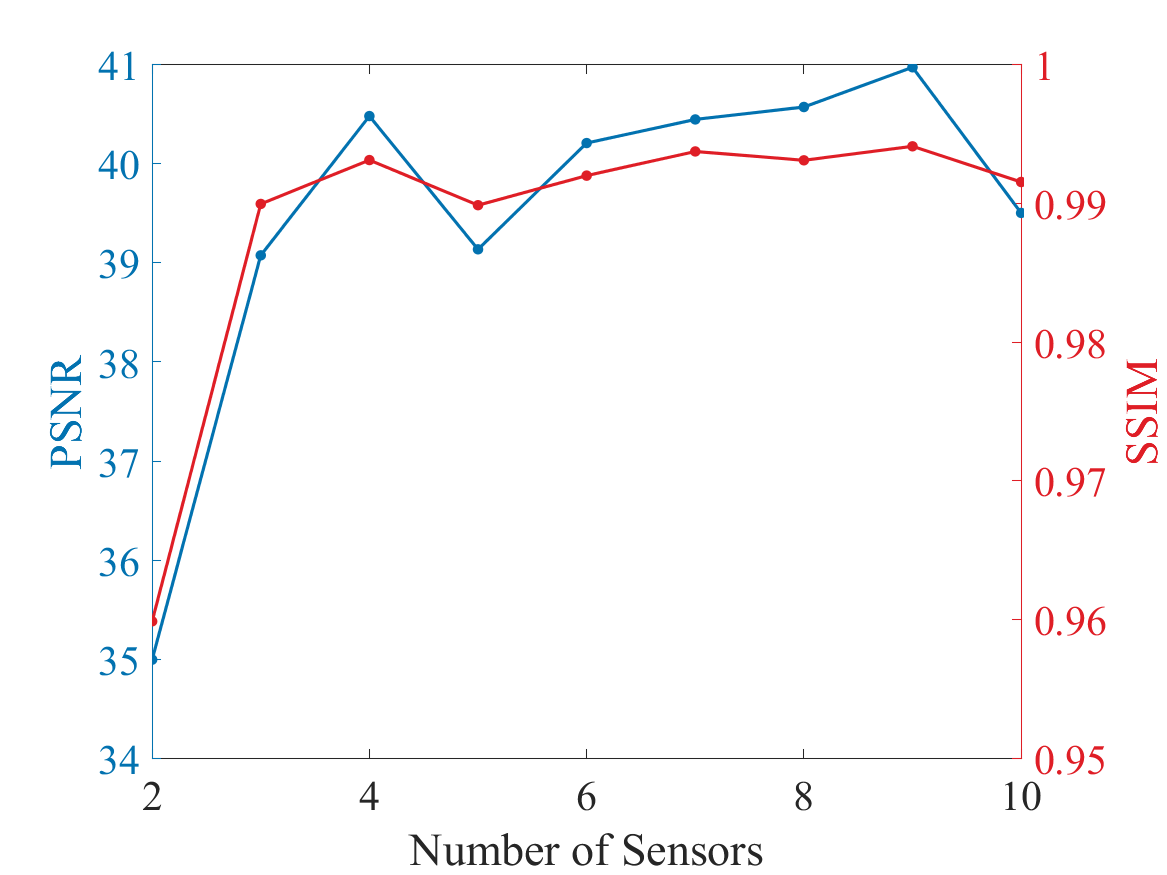}
	\caption{Quantitative recovery results w.r.t. different numbers of sensors on the MSIs \emph{Flowers} (left) and \emph{Toy} (right) at a sampling rate of $10\%$.} \label{fig:sensors}
\end{figure}

\subsubsection{Number of Branch Networks}

DeepONet consist of a trunk network and several branch networks. These branch networks are mainly used to handle the input function. According to \cite{lu2021learning}, using lots of branch networks is inefficient, so all the branch networks are merged into one single neural network, which can be viewed as all the branch networks sharing the same set of parameters.
The number of branch networks (i.e. $P_n$ in Equation \eqref{eq:deeponet}) is the output dimension of this neural network.
Figure \ref{fig:branches} shows the quantitative recovery results w.r.t. different numbers of branch networks. The results show that both PSNR and SSIM values initially rise with an increasing number of branches, reaching an optimal range. Beyond this point, further expanding the branch networks leads to slight fluctuations.

\begin{figure}[!htb]
	\centering
	\includegraphics[width=0.49\linewidth]{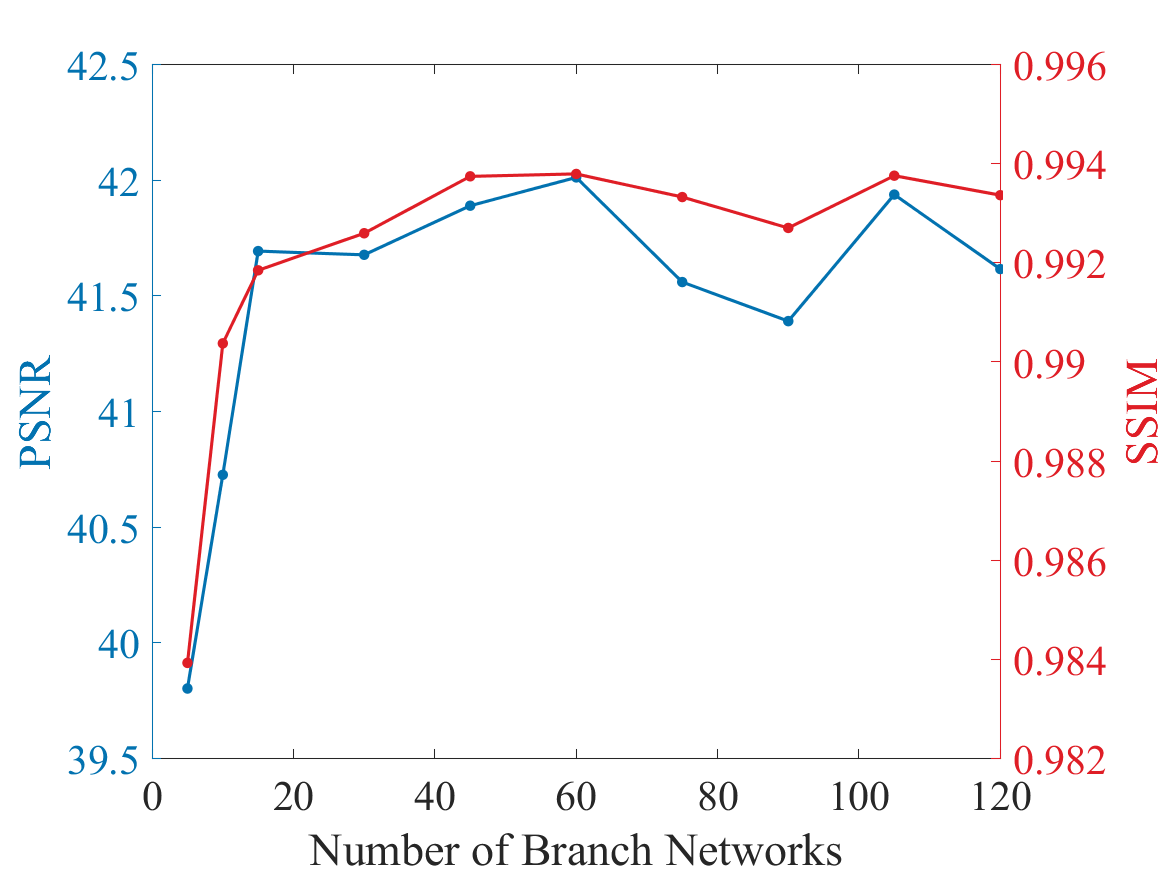}
	\includegraphics[width=0.49\linewidth]{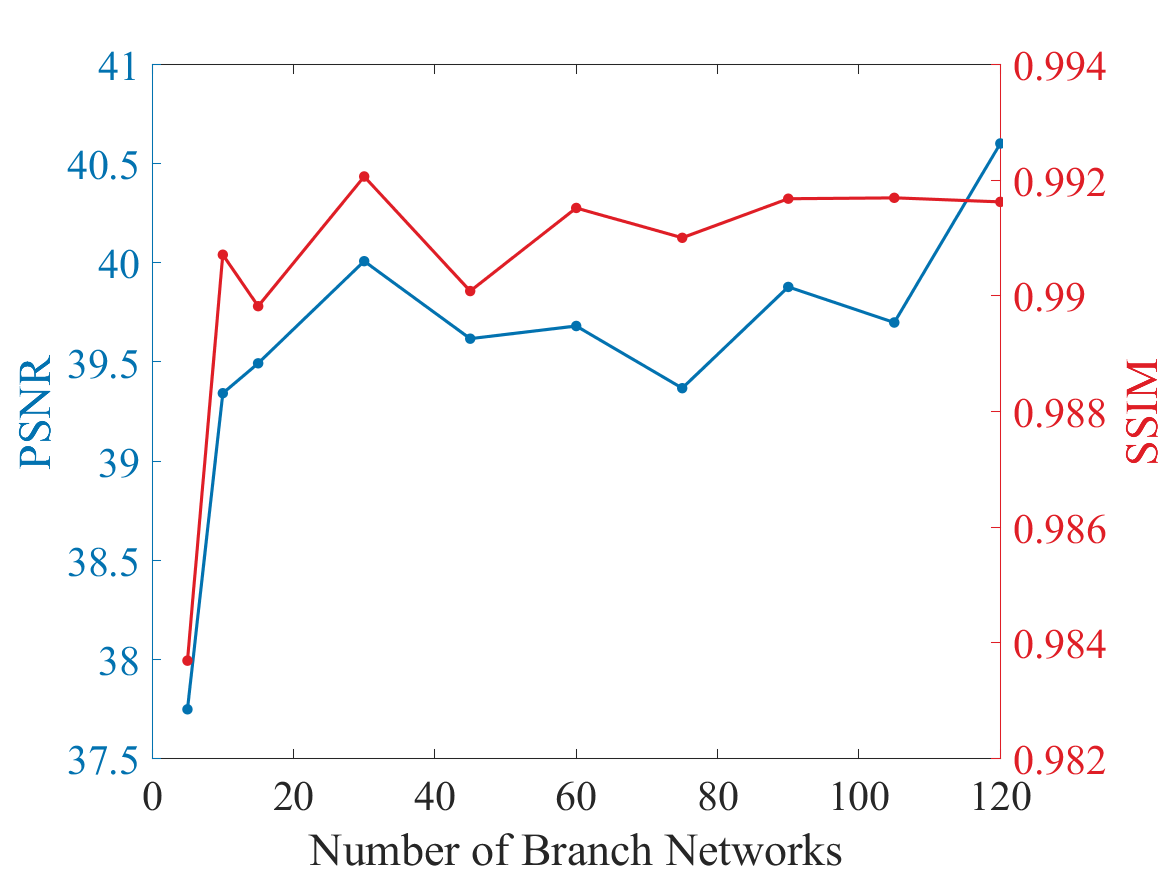}
	\caption{Quantitative recovery results w.r.t. different numbers of branch networks on the MSIs \emph{Flowers} (left) and \emph{Toy} (right) at a sampling rate of $10\%$.} \label{fig:branches}
\end{figure}

\section{Conclusion}

To unlock the potential of continuous tensor function representations, in this paper, we first suggested neural operator-grounded mode-$n$ operators as a continuous and nonlinear alternative, which provides a genuine continuous representation of real-world data and can ameliorate discretization artifacts.
Empowering with continuous and nonlinear mode-$n$ operators, we proposed a neural operator-grounded continuous tensor function representation (abbreviated as NO-CTR), which can more faithfully represent complex real-world data compared with classic discrete tensor representations and continuous tensor function representations.
Theoretically, we also proved that any continuous tensor function can be approximated by NO-CTR. To examine the capability of NO-CTR, we suggested an NO-CTR-based multi-dimensional data completion model.
Extensive experiments across various data on regular mesh grids (multi-spectral images and color videos), on mesh girds with different resolutions (Sentinel-2 images) and beyond mesh grids (point clouds) demonstrated the superiority of NO-CTR.

%\section*{Acknowledgments}
%This should be a simple paragraph before the References to thank those individuals and institutions who have supported your work on this article.

%\nocite{*}

%\section{References Section}
%You can use a bibliography generated by BibTeX as a .bbl file.
% BibTeX documentation can be easily obtained at:
% http://mirror.ctan.org/biblio/bibtex/contrib/doc/
% The IEEEtran BibTeX style support page is:
% http://www.michaelshell.org/tex/ieeetran/bibtex/

 % argument is your BibTeX string definitions and bibliography database(s)
\bibliographystyle{IEEEtran}
\bibliography{./ref}
%
%\section{Simple References}
%You can manually copy in the resultant .bbl file and set second argument of $\backslash${\tt{begin}} to the number of references
% (used to reserve space for the reference number labels box).

%\newpage

%\section{Biography Section}
%If you have an EPS/PDF photo (graphicx package needed), extra braces are
% needed around the contents of the optional argument to biography to prevent
% the LaTeX parser from getting confused when it sees the complicated
% $\backslash${\tt{includegraphics}} command within an optional argument. (You can create
% your own custom macro containing the $\backslash${\tt{includegraphics}} command to make things
% simpler here.)
%
%\vspace{11pt}
%
%{\bf{If you include a photo:}\vspace{-33pt}}
%\begin{IEEEbiography}%[{\includegraphics[width=1in,height=1.25in,clip,keepaspectratio]{fig1}}]
%{Michael Shell}
%Use $\backslash${\tt{begin\{IEEEbiography\}}} and then for the 1st argument use $\backslash${\tt{includegraphics}} to declare and link the author photo.
%Use the author name as the 3rd argument followed by the biography text.
%\end{IEEEbiography}
%
%\vspace{11pt}
%
%{\bf{If you will not include a photo:}\vspace{-33pt}}
%\begin{IEEEbiographynophoto}{John Doe}
%Use $\backslash${\tt{begin\{IEEEbiographynophoto\}}} and the author name as the argument followed by the biography text.
%\end{IEEEbiographynophoto}

\vfill

{\onecolumn \allowdisplaybreaks \appendix[Proof of Theorem \ref{th:approx}]
\begin{lemma}[\cite{calin2020universal}] \label{th:approx of function}
	Let $\sigma$ be any continuous discriminatory function. Then the finite sums of the form
	\[ \sum_{k=1}^K a_k \sigma (\mat{w}_k^T \mat{y} + b_k), \quad \mat{w}_k \in \mathbb{R}^N, \ a_k, b_k \in \mathbb{R} \]
	are dense in $C([0,1]^N)$.
\end{lemma}

\begin{lemma}[\cite{lu2021learning,chen1995universal}] \label{th:approx of operator}
	Suppose that $W$ is a Banach space, $K_1 \subseteq W$ and $K_2 \subseteq \mathbb{R}^N$ are two compact sets in $W$ and $\mathbb{R}^N$, respectively, $V$ is a compact set in $C(K_1)$, $F$ is a continuous operator, which maps $V$ to $C(K_2)$. Then for any $\varepsilon > 0$, there are positive integers $m$ and $P$, continuous functions $b_p \in C(\mathbb{R}^m)$ and $t_p \in C(K_2)$ implemented by deep neural networks, $p=1,2,\cdots,P$, $\mat{x}_1, \mat{x}_2, \cdots \mat{x}_m \in K_1$, such that
	\[ \left| F(v)(\mat{y}) - \sum_{p=1}^{P} b_p(v(\mat{x}_1), v(\mat{x}_2), \cdots, v(\mat{x}_m)) t_p(\mat{y}) \right| < \varepsilon \]
	holds for all $v \in V$ and $\mat{y} \in K_2$.
\end{lemma}

\begin{IEEEproof}[Proof of Theorem \ref{th:approx}]
	According to Lemma \ref{th:approx of function}, there is a two-layer fully connected neural network
	\[ \conttensor{G} \colon [0,1]^N \to \mathbb{R}, \ \mat{y} \mapsto \sum_{k=1}^{K} a_k \sigma(\mat{w}_k^T \mat{y} + b_k) \]
	such that
	\begin{equation} \label{eq:approx of function}
		\| \conttensor{G} - \conttensor{X} \| < \frac{\varepsilon}{2}.
	\end{equation}
	Noticing that $\conttensor{X}$ is a continuous function defined on a compact set $[0,1]^N$, so $\conttensor{X}$ is bounded, and $\conttensor{G}$ is also bounded. Assume that
	\begin{equation} \label{eq:G bound}
		\| \conttensor{G} \| \leqslant M.
	\end{equation}

	According to Lemma \ref{th:approx of operator}, there are $N$ neural operators implemented by DeepONet
	\[ F_n \colon C([0,1]) \to C([0,1]), \ u(\cdot) \mapsto \sum_{p=1}^{P_n} b^{(n)}_p(u(x_1),u(x_2),\cdots,u(x_{m_n})) t^{(n)}_p(\cdot), \]
	such that
	\[ | F_n(u)(y) - I(u)(y) | < \frac{\varepsilon}{(2^{N+1} - 2)M} \]
	holds for all $u \in C([0,1])$ and $y \in [0,1]$, so
	\begin{align*}
		\| F_n - I \| & = \sup_{\|u\|=1} \| (F_n - I)(u) \| \\
		& = \sup_{\|u\|=1} \| F_n(u) - I(u) \| \\
		& = \sup_{\|u\|=1} \sup_{y \in [0,1]} | (F_n(u) - I(u))(y) | \\
		& = \sup_{\|u\|=1} \sup_{y \in [0,1]} | F_n(u)(y) - I(u)(y) | \\
		& \leqslant \frac{\varepsilon}{(2^{N+1} - 2)M},
	\end{align*}
	where $I$ denotes identity transform, $n=1,2,\cdots,N$. Thus,
	\begin{align}
		\| \operator{F}_n^{\langle n \rangle} - \operator{I}^{\langle n \rangle} \| & = \sup_{\|\conttensor{X}\| = 1} \| (\operator{F}_n^{\langle n \rangle} - \operator{I}^{\langle n \rangle}) (\conttensor{X}) \| \notag \\
		& = \sup_{\|\conttensor{X}\| = 1} \sup_{\mat{y} \in [0,1]^N} | (\operator{F}_n^{\langle n \rangle} - \operator{I}^{\langle n \rangle}) (\conttensor{X}) (\mat{y}) | \notag \\
		& = \sup_{\|\conttensor{X}\| = 1} \sup_{\mat{y} \in [0,1]^N} | (F_n - I) (\conttensor{X}(y_1, \cdots, y_{n-1}, \cdot, y_{n+1}, \cdots, y_N)) (y_n) | \notag \\
		& = \sup_{\|\conttensor{X}\| = 1} \sup_{\bar{\mat{y}} \in [0,1]^{N-1}} \sup_{y_n \in [0,1]} | (F_n - I) (\conttensor{X}(y_1, \cdots, y_{n-1}, \cdot, y_{n+1}, \cdots, y_N)) (y_n) | \notag \\
		& = \sup_{\|\conttensor{X}\| = 1} \sup_{\bar{\mat{y}} \in [0,1]^{N-1}} \| (F_n - I) (\conttensor{X}(y_1, \cdots, y_{n-1}, \cdot, y_{n+1}, \cdots, y_N)) \| \notag \\
		& \leqslant \sup_{\|\conttensor{X}\| = 1} \sup_{\bar{\mat{y}} \in [0,1]^{N-1}} \| F_n - I \| \| \conttensor{X}(y_1, \cdots, y_{n-1}, \cdot, y_{n+1}, \cdots, y_N) \| \notag \\
		& = \| F_n - I \| \sup_{\|\conttensor{X}\| = 1} \sup_{\bar{\mat{y}} \in [0,1]^{N-1}} \sup_{y_n \in [0,1]} | \conttensor{X}(y_1, \cdots, y_{n-1}, y_n, y_{n+1}, \cdots, y_N) | \notag \\
		& = \| F_n - I \| \sup_{\|\conttensor{X}\| = 1} \|\conttensor{X}\| \notag \\
		& = \| F_n - I \| \notag \\
		& \leqslant \frac{\varepsilon}{(2^{N+1} - 2)M}, \label{eq:approx of operator}
	\end{align}
	where $\mat{y} = (y_1,y_2,\cdots,y_N) \in [0,1]^N$, $\bar{\mat{y}} = (y_1, \cdots, y_{n-1}, y_{n+1}, \cdots, y_N) \in [0,1]^{N-1}$, $\operator{F}_n^{\langle n \rangle}$ and $\operator{I}^{\langle n \rangle}$ are the continuous and nonlinear mode-$n$ operators induced by $F_n$ and $I$, respectively. Without loss of generality, we assume that
	\begin{equation} \label{eq:Fn bound}
		\| \operator{F}_n^{\langle n \rangle} \| \leqslant 2.
	\end{equation}

	Combining Equations \eqref{eq:approx of function} \eqref{eq:G bound} \eqref{eq:approx of operator} and \eqref{eq:Fn bound}, we obtain
	\begin{align*}
		& \mathrel{\phantom{=}} \| \conttensor{X} - \operator{F}_N^{\langle N \rangle} \circ \operator{F}_{N-1}^{\langle N-1 \rangle} \circ \cdots \circ \operator{F}_1^{\langle 1 \rangle} (\conttensor{G}) \| \\
		& \leqslant \| \operator{I}^{\langle N \rangle} \circ \operator{I}^{\langle N-1 \rangle} \circ \cdots \circ \operator{I}^{\langle 1 \rangle}(\conttensor{X}) - \operator{I}^{\langle N \rangle} \circ \operator{F}_{N-1}^{\langle N-1 \rangle} \circ \cdots \circ \operator{F}_1^{\langle 1 \rangle}(\conttensor{G}) \| \\
		& \mathrel{\phantom{=}}+ \| \operator{I}^{\langle N \rangle} \circ \operator{F}_{N-1}^{\langle N-1 \rangle} \circ \cdots \circ \operator{F}_1^{\langle 1 \rangle}(\conttensor{G}) - \operator{F}_N^{\langle N \rangle} \circ \operator{F}_{N-1}^{\langle N-1 \rangle} \circ \cdots \circ \operator{F}_1^{\langle 1 \rangle}(\conttensor{G}) \| \\
		& = \| \operator{I}^{\langle N \rangle} ( \operator{I}^{\langle N-1 \rangle} \circ \cdots \circ \operator{I}^{\langle 1 \rangle}(\conttensor{X}) - \operator{F}_{N-1}^{\langle N-1 \rangle} \circ \cdots \circ \operator{F}_1^{\langle 1 \rangle}(\conttensor{G}) ) \| + \| (\operator{I}^{\langle N \rangle} - \operator{F}_N^{\langle N \rangle}) (\operator{F}_{N-1}^{\langle N-1 \rangle} \circ \cdots \circ \operator{F}_1^{\langle 1 \rangle}(\conttensor{G})) \| \\
		& \leqslant \| \operator{I}^{\langle N \rangle} \| \| \operator{I}^{\langle N-1 \rangle} \circ \cdots \circ \operator{I}^{\langle 1 \rangle}(\conttensor{X}) - \operator{F}_{N-1}^{\langle N-1 \rangle} \circ \cdots \circ \operator{F}_1^{\langle 1 \rangle}(\conttensor{G}) \| + \| \operator{I}^{\langle N \rangle} - \operator{F}_N^{\langle N \rangle} \| \| \operator{F}_{N-1}^{\langle N-1 \rangle} \circ \cdots \circ \operator{F}_1^{\langle 1 \rangle}(\conttensor{G}) \| \\
		& \leqslant \| \operator{I}^{\langle N-1 \rangle} \circ \cdots \circ \operator{I}^{\langle 1 \rangle}(\conttensor{X}) - \operator{F}_{N-1}^{\langle N-1 \rangle} \circ \cdots \circ \operator{F}_1^{\langle 1 \rangle}(\conttensor{G}) \| + \| \operator{I}^{\langle N \rangle} - \operator{F}_N^{\langle N \rangle} \| \| \operator{F}_{N-1}^{\langle N-1 \rangle} \| \cdots \| \operator{F}_1^{\langle 1 \rangle} \| \| \conttensor{G} \| \\
		& \leqslant \cdots\cdots \\
		& \leqslant \| \operator{I}^{\langle 1 \rangle}(\conttensor{X}) - \operator{F}_1^{\langle 1 \rangle}(\conttensor{G}) \| + \sum_{k=1}^{N-1} \left( \| \operator{I}^{\langle k+1 \rangle} - \operator{F}_{k+1}^{\langle k+1 \rangle} \| \| \conttensor{G} \| \prod_{j=1}^{k} \| \operator{F}_j^{\langle j \rangle} \| \right) \\
		& \leqslant \| \conttensor{X} - \conttensor{G} \| + \| \operator{I}^{\langle 1 \rangle} - \operator{F}_1^{\langle 1 \rangle} \| \| \conttensor{G} \| + \sum_{k=1}^{N-1} \left( \| \operator{I}^{\langle k+1 \rangle} - \operator{F}_{k+1}^{\langle k+1 \rangle} \| \| \conttensor{G} \| \prod_{j=1}^{k} \| \operator{F}_j^{\langle j \rangle} \| \right) \\
		& < \frac{\varepsilon}{2} + \frac{\varepsilon}{(2^{N+1} - 2)M} M + \sum_{k=1}^{N-1} \left( \frac{\varepsilon}{(2^{N+1} - 2)M} M \cdot 2^k \right) \\
		& = \varepsilon.
	\end{align*}
\end{IEEEproof}
}

\end{document}